\documentclass{article}
\usepackage[left=2.5cm, right=2.5cm, top=2.5cm, bottom=2.5cm]{geometry}

\usepackage{times}
\usepackage{soul}
\usepackage{url}
\usepackage[dvipsnames]{xcolor}         
\usepackage[
  colorlinks=true,
  citecolor=MidnightBlue,
  linkcolor=PineGreen,
  urlcolor=Sepia
]{hyperref}
\usepackage[utf8]{inputenc}
\usepackage[small]{caption}
\usepackage{amsmath}
\usepackage{amsthm}
\usepackage{booktabs}
\usepackage[switch]{lineno}

\newtheorem{example}{Example}
\newtheorem{theorem}{Theorem}
\newtheorem{corollary}{Corollary}
\newtheorem{lemma}{Lemma}
\newtheorem{proposition}{Proposition}
\newtheorem{definition}{Definition}
\newtheorem{assumption}{Assumption}

\usepackage{amsfonts}       
\usepackage{nicefrac}       
\usepackage{microtype}      
\usepackage{natbib}

\usepackage{amssymb}
\usepackage{subcaption}
\usepackage{tikz}
\usetikzlibrary{graphs}
\usepackage{tikzsymbols}
\usepackage{enumitem}
\setlist[itemize]{noitemsep, topsep=0pt, leftmargin=*}
\usepackage{pifont}
\newcommand{\cmark}{\textcolor{ForestGreen}{\ding{51}}\xspace}%
\newcommand{\xmark}{\textcolor{BrickRed}{\ding{55}}\xspace}%
\newcommand{\bigO}{\mathcal{O}}

\def\bdelta{{\boldsymbol{\delta}}}
\def\bbeta{{\boldsymbol{\beta}}}
\usepackage{xspace}
\usepackage{bold-extra}
\newcommand{\cdag}{\texttt{\textbf{cDAG}}\xspace}
\usepackage{stmaryrd}
\newcommand{\iset}[1]{\llbracket #1 \rrbracket}
\usepackage[ruled, vlined, linesnumbered]{algorithm2e}
\usepackage[capitalize]{cleveref}
\usepackage{tcolorbox}

\title{
  What makes Models Compositional? \\
  A Theoretical View: 
  With Supplement
}

\usepackage{authblk}
\author[$\dag$]{Parikshit Ram\thanks{\url{parikshit.ram@ibm.com}}}
\author[$\dag$]{Tim Klinger}
\author[$\ddag$]{Alexander G. Gray}
\affil[$\dag$]{IBM Research}
\affil[$\ddag$]{Centaur AI Institute}
\affil[$\ddag$]{Purdue University}

\date{}

\begin{document}

\maketitle

\begin{abstract}
Compositionality is thought to be a key component of language, and various compositional benchmarks have been developed to empirically probe the compositional generalization of existing sequence processing models. These benchmarks often highlight failures of existing models, but it is not clear why these models fail in this way. In this paper, we seek to theoretically understand the role the compositional structure of the models plays in these failures and how this structure relates to their expressivity and sample complexity. We propose a general neuro-symbolic definition of compositional functions and their compositional complexity. We then show how various existing general and special purpose sequence processing models (such as recurrent, convolution and attention-based ones) fit this definition and use it to analyze their compositional complexity. Finally, we provide theoretical guarantees for the expressivity and systematic generalization of compositional models that explicitly depend on our proposed definition and highlighting factors which drive poor empirical performance.
\end{abstract}
\section{Introduction}\label{sec:intro}
Compositionality is assumed to be integral to language processing~\cite{pagin2010compositionalityI,pagin2010compositionalityII}. Generalizing in a compositional manner or {\em compositional generalization} is of high interest when learning with sequences since it can enable a (learned) model to generalize well to a possibly infinite domain of sequences while learning from only a small number of examples.
With this motivation, there has been interest in quantifying the compositional generalization of sequence or language models. This has led to various language modeling benchmarks such as SCAN~\cite{lake2018generalization}, CFQ~\cite{keysers2019measuring}, COGS~\cite{kim2020cogs} and others~\cite{andreas2018measuring,hupkes2020compositionality}, although compositional generalization can also be of interest in vision~\cite{klinger2020study}.

These benchmarks empirically probe the compositionality of  off-the-shelf language models, often demonstrating the lack of compositional generalization. These highlight examples on which the models fail, but there is no precise understanding of why the failures occur. These findings are further put into question by results highlighting how such models can in fact compositionally generalize~\cite{csordas2021devil}. Nonetheless, various novel methods with improved compositional generalization (as measured by these benchmarks) have been developed~\cite{russin2019compositional,gordon2019permutation,li2019compositional,liu2020compositional,nye2020learning,liu2021learning}, utilizing specialized models with compositional inductive biases. 
However, the area of compositional generalization is still {\em lacks a mathematical definition and measure of compositionality}.
\paragraph{Our contributions.}
Inspired by existing discussions on compositionality, and recent solutions for compositional generalization benchmarks, we make the following contributions:
\footnote{A preliminary version of this work~\cite{ram2023how} was previously presented at the KBCG@IJCAI23 workshop.}
\begin{itemize}
\item We propose a general modular definition of ``compositional functions'' to facilitate concrete understanding of the expressiveness and generalization of such functions, and propose the notion of ``compositional complexity'' to quantify the complexity of such functions.
\item We demonstrate the flexibility of this definition by highlighting how various existing models fit this definition, and how complex their compositions are.
\item Given these definitions of compositional functions and compositional complexity, we precisely characterize the expressiveness and systematic generalization of such functions.
\end{itemize}
\section{Related Work}\label{sec:related}
A definition of  compositionality~\cite{pagin2010compositionalityI} states that the meaning $\mu(\cdot)$ of an expression is:
\begin{equation}\label{eq:rec-comp}
\mu(\alpha(u_1, \ldots, u_k)) = r_\alpha(\mu(u_1), \ldots, \mu(u_k)),
\end{equation}
where $\alpha$ is a (grammar) rule applied to the sub-terms $u_i$ to obtain the expression $\alpha(u_1, \ldots, u_k)$, and $r_\alpha$ is a meaning operation that depends on the rule $\alpha$.
A non-technical phrasing of the principle of compositionality~\cite{partee1995lexical} is(quoted from Hupkes {\em et al.}~\citeyear{hupkes2020compositionality}): {\em ``The meaning of a whole is a function of the meanings of the parts and of the way they are syntactically combined.''} 
Among the expected properties of compositional functions are {\bf systematicity} -- the ability to {\em consistently handle unknown combinations of known parts}, and {\bf productivity} -- the ability to {\em handle arbitrary length sequences}. 
For systematic generalization, a model learned from some examples (``known parts'') is expected to generalize to unseen examples (``unknown combinations'').
Productive generalization requires learned models to generalize to sequences longer than those seen during learning.

Understanding compositionality and its relation to systematicity has been of long interest, and Zadrozny~\citeyear{zadrozny1994compositional} showed that compositional semantics can be defined for any meaning function without necessarily inducing systematicity. Recently, Rosenbaum {\em et al.}~\citeyear{rosenbaum2019routing} studied how networks (specifically routing networks) can be composed of multiple modules, focusing mainly on the training dynamics and empirical performance of such compositional models, while Wiedemer {\em et al.}~\citeyear{wiedemer2023compositional} try to formalize compositional generalization from the perspective of disentangled representations. In contrast to these works, we focus on understanding the compositionality of functions in terms of hierarchical computation needed to process the input (sequences), and the kinds of hierarchy induced by existing sequence processing models, such as recurrent, convolutional and attention-based models.

Dziri {\em et al.}~\citeyear{dziri2023faith} empirically probe the compositionality of transformers on problems that require {\em ``strict multi-step computations to derive correct answers''}, and consider the hierarchical computation necessary for the problem solving similar to our definitions.
While Dziri {\em et al.}~\citeyear{dziri2023faith} and others~\cite{kim2020cogs,sikarwar2022transformers,ontanon2022making} show that transformers struggle with such tasks, a separate line of work highlight how transformers can be successful on compositional benchmarks~\cite{csordas2021devil,zhou2023leasttomost,drozdov2023compositional}, which begs the questions (i)~{\em whether} models (such as transformers) are {\em at all capable} of solving compositional tasks, and (ii)~{\em when} such models are {\em able} to solve compositional tasks.
Our proposed framework allows us to study such questions theoretically.

Jarvis {\em et al.}~\citeyear{jarvis2023on} study systematicity for both datasets and functions, showing that modular functions can systematically generalize on datasets with compositional sub-structures, while general purpose non-modular functions do not.
However, if the function modularity is unable to appropriately segregate the compositional sub-structure, even modular functions fail. To that end, they study a space of datasets with cleanly separable systematic sub-structures. 
In the context of \cref{eq:rec-comp}, cleanly-separable systematic sub-structures correspond to non-overlapping sub-terms $u_i, i \in \iset{ k }$. Our proposed framework can be seen as a generalization where the compositional sub-structures may not always be cleanly separable, but still present, and we study general purpose sequence models and their ability to express such compositional sub-structures.
\section{Defining Compositionality} \label{sec:cdefs}
Our goal is to ground this principle into a mathematical form that will allow us to {\em quantify} the compositionality of models and understand how this quantification affects downstream compositional generalization.
We define compositional functions $f: \mathcal X \to \mathcal Y$ with the domain $\mathcal X$ of input sequences $X = \{ x_1, \ldots, x_L \}$ of atoms or tokens $x_i \in \mathcal I$ from an input dictionary $\mathcal I$. The range $\mathcal Y$ of $f$ can be $\mathbb R$ for regression, $\{0, 1\}$ for binary classification, or $\mathcal I$ for next token prediction.
\begin{definition} \label{def:comp-func}
To define $f$, we need the following components:
\begin{itemize}
\item Token encoder $e: \mathcal I \times \mathbb N \to \mathcal H$ (latent space), with $e_i = e(x_i, i) \in \mathcal H$ encoding the $i^{\text{\sf th}}$ token in $X \in \mathcal X$.
\item A computation directed acyclic graph (DAG) or \cdag $D: \mathcal X \to \mathcal D$ (the space of DAGs), with $D(X)$ defining the hierarchical processing of a sequence $X$. $D(X)$ can also be viewed as the trace of program used by function $f$ to process $X$. We will describe this in further detail soon.
\item Span processor $g: \mathcal H^k \to \mathcal H$ maps $k$ terms in the latent space into a new term in the latent space.
\item Read-out function $h: \mathcal H^m \to \mathcal Y$ which maps the final set of terms in the latent space to the output space $\mathcal Y$.
\end{itemize}
\noindent
Given the above, we define a compositional function as
{\begin{equation} \label{eq:comp-def}
f(X) = h \left( g^{\otimes D(X)}(e(x_1,1), \ldots, e(x_L, L)) \right),
\end{equation}}
\noindent
where $g^{\otimes D(X)}$ is the recursive operation of $g$ over $D(X)$.
\end{definition}
Next we further discuss the components.

\begin{figure}[t]
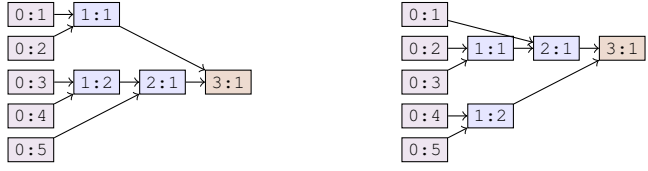

\centering
{\scriptsize {\tt \tikz
  \graph[
    nodes={draw, fill=blue!10},
    grow right sep=0.1in,
    group shift=(-90:0.45),
  ]
  {
    a [as=0:1, fill=Fuchsia!10]
    -> b [as=1:1]
    ;
    c [as=0:2, fill=Fuchsia!10]
    -> b
    ;
    d [as=0:3, fill=Fuchsia!10]
    -> e [as=1:2]
    -> f [as=2:1]
    -> g [as=3:1, fill=Sepia!10]
    ;
    j [as=0:4, fill=Fuchsia!10]
    -> e
    ;
    m [as=0:5, fill=Fuchsia!10]
    -> f
    ;
    b -> g;
  };
}}
\hskip 50pt
{\scriptsize {\tt \tikz
  \graph[
    nodes={draw, fill=blue!10},
    grow right sep=0.1in,
    group shift=(-90:0.45),
  ]
  {
    a [as=0:1, fill=Fuchsia!10];
    c [as=0:2, fill=Fuchsia!10]
    -> b [as=1:1]
    -> f [as=2:1]
    -> g [as=3:1, fill=Sepia!10];
    a -> f;
    d [as=0:3, fill=Fuchsia!10]
    -> b;
    j [as=0:4, fill=Fuchsia!10]
    -> e [as=1:2];
    m [as=0:5, fill=Fuchsia!10]
    -> e;
    e -> g;
  };
}}
\caption{{\cdag}s for $f(X)$ (left) and $f(X')$ (right) in \cref{ex:cfunc-1}. Nodes are labeled $l$:$i$ (level $l$, index $i$). \textcolor{Fuchsia}{Sources are Fuchsia}, \textcolor{Sepia}{sinks are Sepia}, and \textcolor{blue}{internal nodes are Blue}.}
\label{fig:ex:1}
\end{figure}

A {\bf computation DAG} or \cdag $D(X) \triangleq \{N(X), E(X)\}$ for a specific input sequence $X \in \mathcal X$ can depend on $X$ or be pre-specified. This \cdag is a leveled DAG with set of nodes $N(X)$ and edges $E(X)$. Each node $n \triangleq (l:i) \in N(X)$ has a level $l$ and index $i$
(see \cref{asec:cdefs:total-ordering} for details).
The recursive application of $g$ over $D(X)$ induces a value $v_{l:i} \in \mathcal H$ for each internal node $n \in N(X)$. The sources is $N(X)$ have level 0, and there is one source for each $x_i \in X, i \in \iset{ L } \triangleq \{1, \ldots, L\}$ with index $i$ and value $v_{0:i} = e(x_i, i) \in \mathcal H$. There are $m$ sinks in $N(X)$, and at most $k$ incoming edges and $q$ outgoing edges at any node. For an internal node $n \in N(X)$ with $k$ parents $P(n)$, the value $v_{l:i} = g(v_{l_1:i_1}, \ldots, v_{l_k:i_k}) \in \mathcal H$ where $v_{l_j:i_j}$ is the value of the $j^\text{\sf th}$ parent in $P(n)$.
Note that this \cdag corresponds to the ``forward-pass'' for inference.

We consider the explicit \cdag because it allows us to see how the different elements $x_i, i \in \iset{ L }$ of the input sequence $X$ are hierarchically composed to obtain the output. This will allow us to study the complexity of any compositional function. A ``simple'' \cdag, where all source nodes just connect to a single sink node, would be ``applicable'' to all functions, but it does not allow us to study it in an interesting manner. When we study the compositional functions induced by general purpose models (such as recurrent, convolutional or transformer models), we will see that some models have explicit {\cdag}s with more structure, while others have less structured explicit {\cdag}s, but there are implicit structures induced in the \cdag; whenever possible, we will explicitly state this implicit structure and study its properties. From a neuro-symbolic perspective~\cite{sarker2021neuro,garcez2022neural}, this explicit \cdag can be seen as the symbolic part, while the $e, g, h$ are the neural; note that, in some models, this symbolic \cdag might be created with neural elements, while in others, the \cdag might be obtained with a symbolic grammar.
This neuro-symbolic view offers a novel theoretical understanding of compositionality.

The {\bf span processor} $g: \mathcal H^k \to \mathcal H$ takes as input $k$ elements from the latent space $\mathcal H$ and outputs an element in $\mathcal H$. While the definition implies that the same $g$ needs to be operated recursively over the \cdag $D(X)$, there is no restriction on the inputs and output of $g$ regarding the information encoded in the latent space. For example, if the level $l$ of any node $l$:$i$ is encoded into its value $v_{l:i}$, then the $g$ will behave differently across levels ({\em level-dependent}); if the index $i$ of the node $l$:$i$ is encoded into its value, then $g$ will be sensitive to the positional information ({\em order-dependent}); if the value of a node includes the type of the node (for example, a non-terminal in a grammar), then $g$ can be {\em type-dependent}.
Our definition states that the arity of the span processor $g: \mathcal H^k \to \mathcal H$ is $k$. We do so for the ease of exposition, though our definition can incorporate more flexible span processors 
(see \cref{asec:cdefs:g-arity}).

The {\bf read-out function} $h: \mathcal H^m \to \mathcal Y$ finally maps $m$ elements in the latent space to the output space $\mathcal Y$. This separation between $g$ and $h$ was necessary in our proposed definition because we require $g$ to be operable recursively, and thus $g$ can operate in a latent space $\mathcal H$ distinct from $\mathcal Y$. In some applications, $\mathcal H \supseteq \mathcal Y$, in which case, $h$ can be an identity function. 
In an alternate scenario, where the $g$ function is identity, and the \cdag function produces ``trivial {\cdag}s'' -- {\cdag}s where the source nodes are the sink nodes (see \cref{fig:models:flat} for example), $h$ would effectively be a mapping from $\mathcal X \to \mathcal Y$ (subsuming the token encoder within $h$). But in this scenario, we are not able to explicitly view the recursive operation desired in the ``{\em meaning of the whole is a function of the meaning of the parts}'' principle of compositionality. Hence, we make this separation between $h$ and $g$ explicit.
There are couple of aspects of this read-out function we wish to discuss explicitly -- (i)~We assume that $h$ is specifically {\em non-compositional} and processes its input without breaking it up into any sub-problems; we explicitly define the compositional function $f$ separating out $g, D, h$, where $g$ (neural) and $D$ (symbolic) represent the compositional part. (ii)~We require $h$ to have a fixed-arity of $m$ since $g$ and $D$ are aggregating the information over the input.
\begin{example}\label{ex:cfunc-1}
\Cref{fig:ex:1} (left) shows the \cdag $D(X)$ for a compositional $f$ on $X = [x_1, \ldots, x_5]$, with $f(X) = h \left( g \left( g \left( e_1, e_2 \right), g \left( g(e_3, e_4), e_5 \right) \right) \right)$, $k = 2$ in-degree, $q=1$ out-degree, $m=1$ sink, $e_i = e(x_i, i) \in \mathcal H$, span-processor $g: \mathcal H^2 \to \mathcal H$, and read-out function $h: \mathcal H \to \mathcal Y$.
The values $v_{\text{\tt 0:}i} = e_i$ for sources {\tt 0:$i$}, $i \in \{1, \ldots, 5\}$, and the internal node values are: $v_\text{\tt 1:1}\gets g(e_1, e_2)$, $v_\text{\tt 1:2}\gets g(e_3, e_4)$, $v_\text{\tt 2:1}\gets g(v_\text{\tt 1:2}, e_5)$, $v_\text{\tt 3:1}\gets g(v_\text{\tt 1:1}, v_\text{\tt 2:1})$. $h$ operates on $v_\text{\tt 3:1}$ at sink {\tt 3:1}. \Cref{fig:ex:1} (right) shows the \cdag $D(X')$ of the same $f$ on $X' \not= X$ with the same $k=2, q=1, m=1$ and $f(X') = h \left( g( g (e_1, g(e_2, e_3)), g(e_4, e_5) ) \right)$.
\end{example}
\begin{figure}[t]
\centering
{\scriptsize {\tt \tikz
  \graph[
    nodes={draw, fill=blue!10},
    grow right sep=0.1in,
    group shift=(-90:0.45),
  ]
  {
    a [as=0:1, fill=Fuchsia!10]
    -> b [as=1:1]
    -> bb [as=2:1]
    -> bbb [as=3:1, fill=Sepia!10]
    ;
    c [as=0:2, fill=Fuchsia!10]
    -> cc [as=1:2];
    c -> b;
    d [as=0:3, fill=Fuchsia!10]
    -> e [as=1:3]
    -> f [as=2:2]
    -> g [as=3:2]
    -> gg [as=4:1, fill=Sepia!10]
    ;
    d -> b;
    d -> cc;
    j [as=0:4, fill=Fuchsia!10]
    -> jj [as=1:4]
    -> jjj [as=2:3];
    j -> cc;
    jjj -> g;
    jj -> f;
    m [as=0:5, fill=Fuchsia!10]
    -> mm [as=1:5]
    -> mmm [as=2:4];
    mmm -> gg;
    m -> e;
    n [as=0:6, fill=Fuchsia!10]
    -> mm;
    m -> jj;
    n -> jj;
    p [as=0:7, fill=Fuchsia!10]
    -> e;
    p -> mm;
    cc -> bb;
    e -> bb;
    b -> f;
    cc -> jjj;
    mm -> jjj;
    e -> mmm;
    jj -> mmm;
    f -> bbb;
    jjj -> bbb;
    mmm -> g;
    jjj -> gg;
  };
}}
\hskip 50pt
{\scriptsize {\tt \tikz
  \graph[
    nodes={draw, fill=blue!10},
    grow right sep=0.1in,
    group shift=(-90:0.45),
  ]
  {
    a [as=0:1, fill=Fuchsia!10];
    b [as=0:2, fill=Fuchsia!10]
    -> bb [as=1:1]
    -> bbb [as=2:1];
    c [as=0:3, fill=Fuchsia!10];
    d [as=0:4, fill=Fuchsia!10]
    -> dd [as=1:2]
    -> ddd [as=2:2]
    -> dddd [as=3:1]
    -> ddddd [as=4:1, fill=Sepia!10];
    e [as=0:5, fill=Fuchsia!10]
    -> ee [as=1:3]
    -> eee [as=2:3];
    f [as=0:6, fill=Fuchsia!10]
    -> ff [as=1:4]
    -> fff [as=2:4]
    -> ffff [as=3:2]
    -> fffff [as=4:2, fill=Sepia!10];
    g [as=0:7, fill=Fuchsia!10]
    -> gg[as=1:5];
    a -> bb;
    c -> {bb, dd};
    d -> {ee, gg};
    e -> {dd, ff};
    f -> {ee, gg};
    g -> ff;
    bb -> {ddd, fff};
    dd -> {eee, fff};
    ff -> {ddd};
    gg -> {eee};
    {dd, ee} -> {bbb};
    {bbb, fff} -> dddd;
    {gg, ddd} -> ffff;
    {eee, ffff} -> ddddd;
    {dddd, eee} -> fffff;
  };
}}
\caption{{\cdag}s for $\mathsf f(X)$ (left) and $\mathsf f(X')$ (right) in \cref{ex:cfunc-2}.
Nodes are labeled $l$:$i$ (level $l$, index $i$). \textcolor{Fuchsia}{Sources are Fuchsia}, \textcolor{Sepia}{sinks are Sepia}, and \textcolor{blue}{internal nodes are Blue}.
}
\label{fig:ex:2}
\end{figure}

\begin{example}\label{ex:cfunc-2}
\Cref{fig:ex:2} (left) shows the \cdag $\mathsf D(X)$ for a compositional $\mathsf f$ on $X = [x_1, \ldots, x_7]$, with $\mathsf f(X) = \mathsf h \left( v_\text{\tt 4:1}, v_\text{\tt 3:1} \right)$, $k = 3$ maximum in-degree, $q=3$ maximum out-degree, $m=2$ sinks, $e_i = e(x_i, i) \in \mathcal H$, span processor $\mathsf g: \mathcal H^3 \to \mathcal H$, and read-out function $\mathsf h: \mathcal H^2 \to \mathcal Y$.
The source values $v_{\text{\tt 0:}i} = e_i$ for each $i \in \{1, \ldots, 7\}$, and the internal node values are: $v_\text{\tt 1:1}\gets \mathsf g(e_1, e_2, e_3)$, $v_\text{\tt 1:2}\gets \mathsf g(e_2, e_3, e_4)$, $v_\text{\tt 1:3}\gets \mathsf g(e_3, e_5, e_7)$, $v_\text{\tt 1:4}\gets \mathsf g(e_4, e_5, e_6)$, $v_\text{\tt 1:5}\gets \mathsf g(e_5, e_6, e_7)$, $v_\text{\tt 2:1} \gets \mathsf g(v_\text{\tt 1:1}, v_\text{\tt 1:2}, v_\text{\tt 1:3})$, $v_\text{\tt 2:2} \gets \mathsf g(v_\text{\tt 1:1}, v_\text{\tt 1:3}, v_\text{\tt 1:4})$, $v_\text{\tt 2:3} \gets \mathsf g(v_\text{\tt 1:2}, v_\text{\tt 1:4}, v_\text{\tt 1:5})$, $v_\text{\tt 2:4} \gets \mathsf g(v_\text{\tt 1:3}, v_\text{\tt 1:4}, v_\text{\tt 1:5})$, $v_\text{\tt 3:1} \gets \mathsf g(v_\text{\tt 2:1}, v_\text{\tt 2:2}, v_\text{\tt 2:3})$, $v_\text{\tt 3:2} \gets \mathsf g(v_\text{\tt 2:2}, v_\text{\tt 2:3}, v_\text{\tt 2:4})$, $v_\text{\tt 4:1} \gets \mathsf g(v_\text{\tt 3:2}, v_\text{\tt 2:3}, v_\text{\tt 2:4})$.
$\mathsf h$ operates on $v_\text{\tt 3:1}$ and $v_\text{\tt 4:1}$ at sinks {\tt 3:1} and {\tt 4:1}.
\Cref{fig:ex:2} (right) shows the \cdag $\mathsf D(X')$ of the same $\mathsf f$ on $X' \not= X$ with the same $k=3, q=3, m=2$.
\end{example}
While \cref{ex:cfunc-1} is a simple compositional function on a sequence, \cref{ex:cfunc-2} is a more sophisticated one. This is to highlight that our proposed \cref{def:comp-func} can handle functions which require more complex interactions between the tokens in a sequence. \Cref{ex:cfunc-1} has a \cdag with a maximum out-degree $q=1$, implying a single path from any source to a sink. \Cref{ex:cfunc-2} has a \cdag with a maximum out-degree $q=3$ across all levels in the DAG, implying that there can be a large number of paths to any sink from a source. This allows the definition to include functions where certain tokens in the sequence are of much higher importance to the output than others. These examples also highlight that edges in the \cdag are allowed to skip levels, and the sinks can be from different levels, further highlighting the compositional flexibility.

We like to remark on a couple of points here: (i)~Through these examples, we show that our definition explicitly considers how the problem of sequence processing is broken up into sub-problems -- the \cdag embodies how disjoint or intertwined these ``sub-problems'' are by explicitly considering the computation hierarchy. (ii)~For input sequences $X, X'$ from the same problem domain, and the same compositional function $f$, we allow the \cdag to be different --  \cdag $D(X)$ can be input-dependent -- thereby allowing different input sequences to have different sub-problem hierarchies.

Before we discuss properties of this form of compositional functions, we note that {\em such a precise yet flexible definition is one of our contributions}, and we will show how existing models (architectures) fit this definition.
It is a precise elaboration of the succinct recursive \cref{eq:rec-comp} -- we make precise the recursion, and how the sub-terms $u_i$ are recursively built up. 
At a non-technical level, we also believe that our proposed \cref{def:comp-func} connects intuitively to existing definitions:
\begin{equation*}
\begin{split}
\underbrace{\text{\footnotesize The meaning of the whole}}_{f:\mathcal X \to \mathcal Y}
\text{\footnotesize is a}
\underbrace{\text{\footnotesize function}}_{h:\mathcal H^m \to \mathcal Y}
\text{\footnotesize of}
\underbrace{\text{\footnotesize the meanings of the parts}}_{g:\mathcal H^k \to \mathcal H}
\\
\text{\footnotesize and}
\underbrace{\text{\footnotesize of the way they are syntactically combined.}}_{D:\mathcal X \to \mathcal D}
\end{split}
\end{equation*}
Both \cref{ex:cfunc-1,ex:cfunc-2} can be seen as compositional functions, but \cref{ex:cfunc-2} is clearly a more complex composition. In addition to its intuitive nature, our proposed definition allows us to understand {\em how complex the compositionality is} beyond just stating if a function is compositional.
\paragraph{Compositional Complexity.}
This depends on the functions $g, h, e$ as well as the \cdag function $D$ that drives the computation. For a sequence $X$ of length $L$, $D(X)$ has $L$ source nodes, maximum in-degree of $k$ (controlling the span size for $g$), $m$ sink nodes (controlling the capacity of $h$), maximum out-degree of $q$ (quantifying the ``localism'' of the effect of a node). However, these do not explicitly incorporate the fact that changes to nodes at lower levels of the {\cdag} {\em can} have a larger effect on the output than changes to nodes at higher levels of the \cdag.
Instead, we propose a new quantification -- the {\em locus of influence} or LoI of any source node:
\begin{definition}[LoI of a source node] \label{def:loi}
Consider a compositional function $f$ with components $e, D, g, h$ (as in \cref{def:comp-func}). Let $(v_{n_1}, \ldots, v_{n_j}, \ldots, v_{n_k}) \in \mathcal H^k$ be any input to the span processor $g$, with $v_{n} = g(v_{n_1}, \ldots, v_{n_j}, \ldots, v_{n_k})$ its output. Let $\varepsilon \in \mathcal H$ be a ``perturbation'' to the $j^{\text{th}}$ argument to $g$, $j \in \llbracket k \rrbracket$, resulting in the perturbed output $v_{n}^j(\varepsilon) = g(v_{n_1}, \ldots, v_{n_j} + \varepsilon, \ldots, v_{n_k})$. Let $c>0$ be an universal constant such that $\forall j \in \llbracket k \rrbracket$, $\forall \varepsilon \in \mathcal H$,
\begin{equation}
\left \| v_{n} - v_{n}^j(\varepsilon)
\right \| \leq  c \| \varepsilon \|.
\end{equation}
For a sequence $X \in \mathcal X$ of length $L$, and a source node {\tt 0:$i$} in $D(X)$, let $P(x_i)$ be the set of all unique paths from {\tt 0:$i$} to any of the sink nodes in $D(X)$. We define the absolute LoI of index $i$ as $\delta_i = \sum_{P \in P(x_i)} c^{|P|}$, with $|P|$ as the length of a path $P \in P(x_i)$, and the relative LoI as $\beta_i = \delta_i / \sum_{j \in \llbracket L \rrbracket} \delta_j$.
\end{definition}

This definition of the complexity of composition incorporates both the complexity of the \cdag $D(X)$ and the complexity of the span processor $g: \mathcal H^k \to \mathcal H$ in terms of its smoothness, with higher values of $c$ indicating more complex (less smooth) $g$. The absolute LoI $\delta_i$ incorporates the effect of longer paths, with the effect growing with path length, and corresponds to the sensitivity of the compositional function output to any one input token in the sequence.

The smaller the absolute LoI $\delta_i$ of any input index $i$, more local its effect, and thus more structure that can be transferred between examples if $x_i$ is replaced with something else. A relative LoI $\beta_i$ greater than $1/L$ denotes that the input index $i$ (and thus input token $x_i$) has an out-sized effect on $D(X)$ (and thus the computation) compared to the other indices (tokens). In \cref{ex:cfunc-1} (left), $\delta_1 = c^2, \beta_1 = \nicefrac{1}{2c+3} < \nicefrac{1}{5}$ while $\delta_3 = c^3,  \beta_3 = \nicefrac{c}{2c+3} > \nicefrac{1}{5}$, implying that $x_3$ has more influence (absolute and relative) function than $x_1$ (assuming $c>1$). In \cref{ex:cfunc-2} (left), $\delta_1 = c^4 + 2c^3, \beta_1 = \nicefrac{c+2}{27c+39} \approx \nicefrac{1}{22} < \nicefrac{1}{7}$, while $\delta_5 = 7c^4 + 9c^3, \beta_5 = \nicefrac{7c+9}{27c+39} \approx \nicefrac{1}{4} > \nicefrac{1}{7}$, hence $x_5$ has a significantly larger influence than $x_1$.

We utilize the LoI to define the complexity of a compositional function, and a class of such compositional functions:
\begin{definition} \label{def:func-class}
A function $f: \mathcal X \to \mathcal Y$ with components $g, h, e, D$ is $(k, q, m, \bdelta, \bbeta)$-compositional if, for any $X \in \mathcal X$ of length $L$ (that is, $|X|=L$), the \cdag $D(X)$ has a  in-degree of $k$, maximum outgoing degree of $q$,  and $m$ sink nodes, and for $\forall i \in \iset{ L }, \delta_i \leq \bdelta$, and $\beta_i \leq \bbeta \in [1/L, 1)$. We denote with $\mathcal F$ a class of such $(k, q, m, \bdelta, \bbeta)$-compositional functions.
\end{definition}
A small $\bdelta$ and a $\bbeta$ close to $1/L$ signifies a function that possesses a high level of localism across all input sequences and tokens in its domain. While this function has the most structure, it might not be suitable for practical purposes.
A high $\bdelta$ and a $\bbeta$ close to $1/L$ signifies a very complex function where there is a lot of interaction between all the input tokens in all input sequences, making it hard to exploit any compositional structure in the function.
A high $\bdelta$ and a $\bbeta$ significantly higher than $1/L$ indicates an interesting class of functions where, some input tokens {\em can have a high influence over the function computation}, but, for most tokens, there is a compositional structure in the function that can be exploited. This intuitively seems to be an interesting and more practical class of compositional functions since assuming all tokens have an equal level of relative influence seems quite restrictive.
\paragraph{Cleanly-separable systematic sub-structures.}
Revisiting \cref{eq:rec-comp}, where $u_i$ are sub-parts, we highlight a special case of our proposed \cref{def:func-class} of compositional functions: If the sub-parts $u_i, u_j, i \not=j$ are non-overlapping throughout the recursion, up to the base case where the sub-parts, $u_i$, are the tokens in the input, then it would induce a \cdag with a maximum outgoing degree $q = 1$. This implies that the number of paths $|P(x_i)| = 1$ for any source {\tt 0}:{$i$} -- there is a single source-to-sink path for any source, resulting in a \cdag with significantly reduced compositional complexity.
\section{Existing Models as Compositional Functions} \label{sec:models}
Here we will discuss various model classes induced by existing model architectures, how they fit \cref{def:comp-func} of compositional functions, and how they compare to each other. 
We will re-express existing sequence processing models as per our definition, teasing out the symbolic \cdag (and the neural $g, h$) and studying their compositional complexity.
The presented \cdag for each model class corresponds to a ``forward-pass'' for inference.
Omitted technical details are in \cref{asec:exp}.
\begin{figure*}[t]
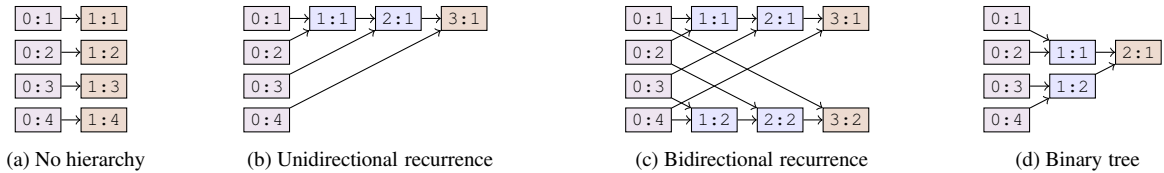

\centering
\begin{subfigure}{0.15\textwidth}
\centering
{\scriptsize {\tt \tikz
\graph[
  nodes={draw, fill=blue!10},
  grow right sep=0.1in,
  group shift=(-90:0.45),
]
{
  a[as=0:1, fill=Fuchsia!10] -> aa[as=1:1, fill=Sepia!10];
  b[as=0:2, fill=Fuchsia!10] -> bb[as=1:2, fill=Sepia!10];
  c[as=0:3, fill=Fuchsia!10] -> cc[as=1:3, fill=Sepia!10];
  d[as=0:4, fill=Fuchsia!10] -> dd[as=1:4, fill=Sepia!10];
};
}}
\caption{No hierarchy}
\label{fig:models:flat}
\end{subfigure}
~
\begin{subfigure}{0.29\textwidth}
\centering
{\scriptsize {\tt \tikz
\graph[
  nodes={draw, fill=blue!10},
  grow right sep=0.1in,
  group shift=(-90:0.45),
]
{
  a[as=0:1, fill=Fuchsia!10]
  -> aa[as=1:1]
  -> aaa[as=2:1]
  -> aaaa[as=3:1, fill=Sepia!10];
  b[as=0:2, fill=Fuchsia!10] -> aa;
  c[as=0:3, fill=Fuchsia!10] -> aaa;
  d[as=0:4, fill=Fuchsia!10] -> aaaa;
};
}}
\caption{Unidirectional recurrence}
\label{fig:models:unirnn}
\end{subfigure}
~
\begin{subfigure}{0.29\textwidth}
\centering
{\scriptsize {\tt \tikz
\graph[
  nodes={draw, fill=blue!10},
  grow right sep=0.1in,
  group shift=(-90:0.45),
]
{
  a[as=0:1, fill=Fuchsia!10]
  -> aa[as=1:1]
  -> aaa[as=2:1]
  -> aaaa[as=3:1, fill=Sepia!10];
  b[as=0:2, fill=Fuchsia!10] -> aa;
  c[as=0:3, fill=Fuchsia!10] -> aaa;
  d[as=0:4, fill=Fuchsia!10]
  -> ee[as=1:2]
  -> eee[as=2:2]
  -> eeee[as=3:2, fill=Sepia!10];
  d -> aaaa;
  c -> ee;
  b -> eee;
  a -> eeee;
};
}}
\caption{Bidirectional recurrence}
\label{fig:models:birnn}
\end{subfigure}
~
\begin{subfigure}{0.2\textwidth}
\centering
{\scriptsize {\tt \tikz
\graph[
  nodes={draw, fill=blue!10},
  grow right sep=0.1in,
  group shift=(-90:0.45),
]
{
  a[as=0:1, fill=Fuchsia!10];
  b[as=0:2, fill=Fuchsia!10] -> aa[as=1:1] -> aaa[as=2:1, fill=Sepia!10];
  a -> aa;
  c[as=0:3, fill=Fuchsia!10]
  -> cc[as=1:2]
  -> aaa;
  d[as=0:4, fill=Fuchsia!10]
  -> cc;
};
}}
\caption{Binary tree}
\label{fig:models:bintree}
\end{subfigure}
\caption{Existing models - I. Nodes are labeled $l$:$i$ (level $l$, index $i$). \textcolor{Fuchsia}{Sources are Fuchsia}, \textcolor{Sepia}{sinks are Sepia}, and \textcolor{blue}{internal nodes are Blue}.}
\label{fig:models-1}
\end{figure*}

We begin with the trivial \cdag of {\bf no hierarchical composition} in \cref{fig:models:flat}. All sources in $D(X)$ connect to separate sinks. The composition is written as $h(g(e_1), g(e_2), g(e_3), g(e_4))$. For $L$-length inputs, this is a $(k = 1, q = 1, m = L, \bdelta = c, \bbeta = \nicefrac{1}{L})$-compositional function class. 
\subsection{Recurrent Composition}\label{sec:models:rec}
\Cref{fig:models:unirnn} presents an example \cdag for  {\bf unidirectional recurrent composition}, with the corresponding compositional function $h(g(g(g(e_1, e_2), e_3), e_4))$. Specific choices of the $g$ and $h$ functions would give us specific recurrent neural network (RNN) models.
The \cdag recursively combines 2 nodes in the order of the original sequence to get a class of compositional functions with $(k=2, q=1, m=1)$. 
The \cdag is input-agnostic, that is $D(X) = D(X') \forall X, X' \in \mathcal X$ with $|X|=|X'|$.
This composition can operate on arbitrary length sequences, and its complexity is quantified as:
\begin{proposition}\label{prop:model:unirnn}
With unidirectional recurrent composition, the maximum absolute LoI is $\bdelta \triangleq c^{L-1}$, with a maximum relative LoI of $\bbeta \triangleq (c^L - c^{L-1})/(2c^L -c^{L-1} - 1)$.
\end{proposition}
The absolute LoI is large (exponential in the input length), and the relative LoI is close to $\nicefrac{1}{2} \gg \nicefrac{1}{L}$ for large enough $L$.

\Cref{fig:models:birnn} is an example of a \cdag of a {\bf bidirectional recurrent composition}, which can be written algebraically as $h(g(g(g(e_1, e_2), e_3), e_4), g(g(g(e_4, e_3), e_2), e_1) )$. The \cdag recursively combines 2 nodes in the order of the sequence, then in the reverse order, giving us a function class with $(k = 2, q = 2, m = 2)$. This is similar to the unidirectional recurrence, but with two sink nodes instead of one, and can operate on arbitrary length sequences but the \cdag is input-agnostic. The compositional complexity is given by:
\begin{proposition}\label{prop:model:birnn}
With bidirectional recurrent composition, the maximum absolute LoI is $\bdelta \triangleq c^{L-1} + c$, with a maximum relative LoI of $\bbeta \triangleq (c^L - c^{L-1} + c^2 - c)/2(2c^L -c^{L-1} - 1)$.
\end{proposition}
Note that, while the $\bdelta$ for the bidirectional recurrent composition remains of the same order as that of the unidirectional recurrent one, $\bbeta$ is approximately halved (approaching $1/4$).

\Cref{fig:models:bintree} shows an example of a {\bf balanced tree recurrent composition} by utilizing a balanced binary-tree \cdag with $(k=2, q=1, m=1)$ (as in a TreeLSTM~\cite{tai2015improved}), with an algebraic form $h(g(g(e_1, e_2), g(e_3, e_4)))$. The \cdag is input-agnostic (like the previous two), but can operate on arbitrary length sequences. The compositional complexity is:
\begin{proposition}\label{prop:model:treernn}
With balanced binary-tree recurrent composition, the maximum absolute LoI is $\bdelta \triangleq c^{\lceil \log_2 L \rceil}$, with a maximum relative LoI of $\bbeta \triangleq 1/L$.
\end{proposition}
This form of composition significantly reduces the complexity of the \cdag relative to the previous two models both in terms of $\bdelta$ (linear dependence in $L$ instead of exponential) and $\bbeta$ ($1/L$ instead of a constant $\gg 1/L$).
There are versions of the tree recurrent composition that leverage the parse tree of the input to define the \cdag~\cite{socher2010learning}, and thus, natively fit our definition. Here the \cdag will no longer be input-agnostic, and \cref{prop:model:treernn} would not apply; the complexity will depend on the grammar driving the input-dependent parse trees. Various models~\cite{bowman2016fast,shen2018neural,shen2019ordered} integrate parse-tree structures into a RNN for an input-dependent \cdag for the recurrent composition. They learn to generate a parse-tree for any given input (either via supervision from an external parser or directly from the data). All these models fit our definition of compositional functions with an input-dependent tree-based \cdag.
\footnote{To the best of our knowledge, none of these architectures have been evaluated with compositional generalization benchmarks, although the motivations for these models leveraging parse-trees align well with the goals of compositional generalization.}
\begin{figure*}[t]
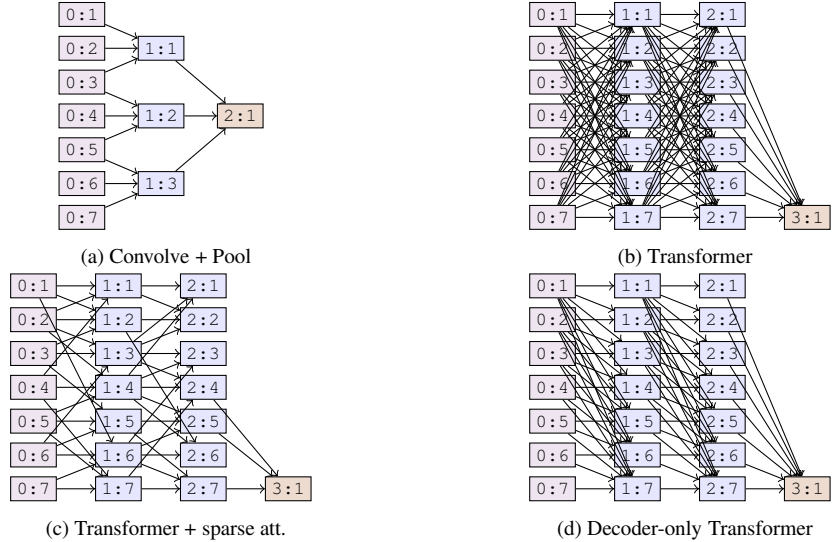

\centering
\begin{subfigure}{0.4\textwidth}
\centering
{\scriptsize {\tt \tikz
\graph[
  nodes={draw, fill=blue!10},
  grow right sep=0.17in,
  group shift=(-90:0.45),
]
{
  a[as=0:1, fill=Fuchsia!10];
  b[as=0:2, fill=Fuchsia!10]
  -> bb [as=1:1];
  a -> bb;
  c[as=0:3, fill=Fuchsia!10];
  c -> bb;
  d[as=0:4, fill=Fuchsia!10]
  -> dd[as=1:2]
  -> ddd[as=2:1, fill=Sepia!10];
  c -> dd;
  bb -> ddd;
  e[as=0:5, fill=Fuchsia!10];
  e -> dd;
  f[as=0:6, fill=Fuchsia!10]
  -> ff[as=1:3];
  e -> ff;
  ff -> ddd;
  g[as=0:7, fill=Fuchsia!10];
  g -> ff;
};
}}
\caption{Convolve + Pool}
\label{fig:models:conv}
\end{subfigure}
~
\begin{subfigure}{0.4\textwidth}
\centering
{\scriptsize {\tt \tikz
\graph[
  nodes={draw, fill=blue!10},
  grow right sep=0.2in,
  group shift=(-90:0.45),
]
{
  a[as=0:1, fill=Fuchsia!10]
  -> a1[as=1:1]
  -> a2[as=2:1];
  b[as=0:2, fill=Fuchsia!10]
  -> b1[as=1:2]
  -> b2[as=2:2];
  c[as=0:3, fill=Fuchsia!10]
  -> c1[as=1:3]
  -> c2[as=2:3];
  d[as=0:4, fill=Fuchsia!10]
  -> d1[as=1:4]
  -> d2[as=2:4];
  e[as=0:5, fill=Fuchsia!10]
  -> e1[as=1:5]
  -> e2[as=2:5];
  f[as=0:6, fill=Fuchsia!10]
  -> f1[as=1:6]
  -> f2[as=2:6];
  g[as=0:7, fill=Fuchsia!10]
  -> g1[as=1:7]
  -> g2[as=2:7]
  -> g3[as=3:1, fill=Sepia!10];
  a -> {b1, c1, d1, e1, f1, g1};
  b -> {a1, c1, d1, e1, f1, g1};
  c -> {b1, a1, d1, e1, f1, g1};
  d -> {b1, c1, a1, e1, f1, g1};
  e -> {b1, c1, d1, a1, f1, g1};
  f -> {b1, c1, d1, e1, a1, g1};
  g -> {b1, c1, d1, e1, f1, a1};
  a1 -> {b2, c2, d2, e2, f2, g2};
  b1 -> {a2, c2, d2, e2, f2, g2};
  c1 -> {b2, a2, d2, e2, f2, g2};
  d1 -> {b2, c2, a2, e2, f2, g2};
  e1 -> {b2, c2, d2, a2, f2, g2};
  f1 -> {b2, c2, d2, e2, a2, g2};
  g1 -> {b2, c2, d2, e2, f2, a2};
  {a2, b2, c2, d2, e2, f2} -> g3;
};
}}
\caption{Transformer}
\label{fig:models:trf}
\end{subfigure}
~
\begin{subfigure}{0.4\textwidth}
\centering
{\scriptsize {\tt \tikz
\graph[
  nodes={draw, fill=blue!10},
  grow right sep=0.2in,
  group shift=(-90:0.45),
]
{
  a[as=0:1, fill=Fuchsia!10]
  -> a1[as=1:1]
  -> a2[as=2:1];
  b[as=0:2, fill=Fuchsia!10]
  -> b1[as=1:2]
  -> b2[as=2:2];
  c[as=0:3, fill=Fuchsia!10]
  -> c1[as=1:3]
  -> c2[as=2:3];
  d[as=0:4, fill=Fuchsia!10]
  -> d1[as=1:4]
  -> d2[as=2:4];
  e[as=0:5, fill=Fuchsia!10]
  -> e1[as=1:5]
  -> e2[as=2:5];
  f[as=0:6, fill=Fuchsia!10]
  -> f1[as=1:6]
  -> f2[as=2:6];
  g[as=0:7, fill=Fuchsia!10]
  -> g1[as=1:7]
  -> g2[as=2:7]
  -> g3[as=3:1, fill=Sepia!10];
  a -> {b1, f1};
  b -> {a1, c1, d1};
  c -> {b1, e1};
  d -> {a1, g1};
  e -> {d1, g1};
  f -> {c1, e1};
  g -> f1;
  a1 -> b2;
  b1 -> f2;
  c1 -> {a2, b2, d2};
  d1 -> {a2, c2, e2, f2};
  e1 -> g2;
  f1 -> {c2, e2, g2};
  g1 -> d2;
  {e2, d2} -> g3;
};
}}
\caption{Transformer + sparse att.}
\label{fig:models:trf-hat}
\end{subfigure}
~
\begin{subfigure}{0.4\textwidth}
\centering
{\scriptsize {\tt \tikz
\graph[
  nodes={draw, fill=blue!10},
  grow right sep=0.2in,
  group shift=(-90:0.45),
]
{
  a[as=0:1, fill=Fuchsia!10]
  -> a1[as=1:1]
  -> a2[as=2:1];
  b[as=0:2, fill=Fuchsia!10]
  -> b1[as=1:2]
  -> b2[as=2:2];
  c[as=0:3, fill=Fuchsia!10]
  -> c1[as=1:3]
  -> c2[as=2:3];
  d[as=0:4, fill=Fuchsia!10]
  -> d1[as=1:4]
  -> d2[as=2:4];
  e[as=0:5, fill=Fuchsia!10]
  -> e1[as=1:5]
  -> e2[as=2:5];
  f[as=0:6, fill=Fuchsia!10]
  -> f1[as=1:6]
  -> f2[as=2:6];
  g[as=0:7, fill=Fuchsia!10]
  -> g1[as=1:7]
  -> g2[as=2:7]
  -> g3[as=3:1, fill=Sepia!10];
  a -> {b1, c1, d1, e1, f1, g1};
  b -> {c1, d1, e1, f1, g1};
  c -> {d1, e1, f1, g1};
  d -> {e1, f1, g1};
  e -> {f1, g1};
  f -> g1;
  a1 -> {b2, c2, d2, e2, f2, g2};
  b1 -> {c2, d2, e2, f2, g2};
  c1 -> {d2, e2, f2, g2};
  d1 -> {e2, f2, g2};
  e1 -> {f2, g2};
  f1 -> g2;
  {a2, b2, c2, d2, e2, f2} -> g3;
};
}}
\caption{Decoder-only Transformer}
\label{fig:models:dotrf}
\end{subfigure}
\caption{Existing models - II. Nodes are labeled $l$:$i$ (level $l$, index $i$). \textcolor{Fuchsia}{Sources are Fuchsia}, \textcolor{Sepia}{sinks are Sepia}, and \textcolor{blue}{internal nodes are Blue}.}
\label{fig:models-2}
\end{figure*}
\subsection{Convolutional Composition with Pooling}\label{sec:models:convpool}
\Cref{fig:models:conv} shows an example \cdag induced by repeated application of {\bf convolution-then-pooling} or conv+pool, performing a convolution over a span of size 2, and pooling over a span of size 2 in this example 
(see \cref{afig:models:conv} for a detailed version of \cref{fig:models:conv} that shows how repeated conv+pool implicitly induces the \cdag in \cref{fig:models:conv}; 
\cref{afig:models:conv2} and corresponding \cdag in \cref{afig:models:conv2-cnd} provides an example of convolution with padding over a span of size 2 and pooling over a span of 3).
The composition in \cref{fig:models:conv} can be algebraically written as $h(g(g(e_1, e_2, e_3), g(e_3, e_4, e_5), g(e_5, e_6, e_7)))$. In general, we can consider a span processor $g$ function that performs convolution over a span of size $w$ and then pools over a span of $p$. Here, the number of sinks $m$ has to be user-specified, and the \cdag repeatedly applies conv+pool until it reduces the number of nodes at the highest level to $m$. This induces a compositional function class with $(k=w+p-1, q = w)$.
\begin{proposition}\label{prop:model:convpool}
Assuming that $1 < w, p \ll L$, the conv+pool composition has a maximum absolute LoI of $\bdelta \sim \bigO(c^{\log L})$, and a maximum relative LoI of $\bbeta \sim \bigO(2/(L (1 + 1/p)))$.
\end{proposition}
We would like to highlight an important subtlety here: If we utilize average or sum pooling, the \cdag is input-agnostic; if the pooling is selective, such as in max-pooling or min-pooling, {\em some of the edges in the \cdag implicitly get deactivated during the recursive computation of the function, leading to an input-dependent \cdag} -- the \cdag itself depends on the input sequence, and can change between inputs of same length. \Cref{afig:models:conv-p1} shows an example of selective pooling, with the corresponding input-dependent \cdag shown in \cref{afig:models:conv-p1-cnd}. 
This function can operate on arbitrary length input sequences.
\subsection{Transformers} \label{sec:models:trf}
\Cref{fig:models:trf} expresses the \cdag of a {\bf transformer} with $M$ levels or ``blocks'', where the span-processor $g: \mathcal H^k \to \mathcal H$ is a transformer block with the (multi-headed) self-attention, residual connections, layer normalization, and the token-wise ReLU network~\cite{vaswani2017attention} 
(see \cref{asec:model:trf-g}).
We consider the version where, for inference, only the last token representation (after $M$ transformer blocks) is utilized for prediction 
\footnote{The last token can be a {\tt [CLS]} token for aggregation.} 
-- the read-out function only applies to the last token representation. For a input sequence of length $L$, $k = q = L$, with $m=1$ sink node.
The compositional complexity is:
\begin{proposition}\label{prop:model:trf}
A transformer based composition with $M$ blocks has a the maximum absolute LoI of $\bdelta = {L}^{M+1} c^{M+1}$, and a maximum relative LoI of $\bbeta = 1/L$.
\end{proposition}
This function class has a really high $\bdelta$ (exponential in the number of blocks $M$ but polynomial in the input length $L$ unless $M > L$) and the lowest possible $\bbeta = 1/L$ -- the output is equally sensitive to all tokens in the sequence. It is interesting to contrast this complexity to that of a balanced binary-tree recurrent composition, which has the same level of relative LoI $\bbeta = 1/L$, but the $\bdelta$ is linear in $L$ instead of polynomial (or exponential if $M > L$); the unidirectional and bidirectional recurrent compositions have an exponential dependence on the length $L$ (but much larger $\bbeta$).

This composition does not explicitly have an input-dependent {\cdag} -- all nodes at a level always connect to all nodes at the next level since all attention (scores) are positive. In practice, the attention weights are obtained with softmax activation, which pushes the attention to be approximately sparse. If we consider the attention to be exactly sparse~\cite{tay2022efficient} with $K \ll L$ nonzero scores (rest zero), the \cdag will have much fewer edges as in \cref{fig:models:trf-hat} (here, $K=3$). If the attention weights and sparsity pattern are input-dependent (as in with top-$K$ attention~\cite{gupta2021memory}), transformers with hard attention will have input-dependent {\cdag}s.
Here $k = K, q = L, m = 1$, and the compositional complexity is:
\begin{proposition}\label{prop:model:trf-att}
A transformer with $M$ blocks and $K$-sparse attention ($K\ll L$) has a maximum absolute LoI of $\bdelta = L K^M c^{M+1}$, and a maximum relative LoI of $\bbeta = 1/K$.
\end{proposition}
Contrasting this to the complexity of the standard transformer (\cref{prop:model:trf}), we see that the $\bdelta$ is considerably lower ($\bigO(L{L}^M)$ vs $\bigO(L K^M)$ with $K \ll L$), implying that the LoI of any one input token (thus the sensitivity) is substantially reduced. However, the relative LoI $\bbeta$ can be considerably higher ($1/L$ vs $1/K$), indicating that the composition is allowed to be relatively more sensitive to certain tokens than others.
Attention-based span-processor $g$ can handle arbitrary length inputs (albeit at quadratic computation per block/level), and the read-out function $h$ always has a fixed arity of $m = 1$. 
Thus, this transformer can operate on arbitrary length inputs.
\footnote{One practical length limitation with transformers is the (learned) absolute positional encoding. However, that can be mitigated with the use of relative positional encoding~\cite{shaw2018self}.}

If we consider a {\bf decoder-only transformer}, then for all levels $l < M$, all edges from nodes $l$:$i \to l$+$1$:$j$ for $j < i$ would no longer be in the \cdag (\cref{fig:models:dotrf}). The complexity is:
\begin{proposition}\label{prop:model:dotrf}
A decoder-only transformer with $M$ blocks has a maximum absolute LoI is $\bdelta = L^M c^{M+1}$, and a maximum relative LoI of $\bbeta 
= 1 / (1 + \sum_{i \in \iset{ L-1 }} (\nicefrac{i}{L})^M )$.
\end{proposition}
As $L$ or $M$ grows, the sum $\sum_i (\nicefrac{i}{L})^M$ goes to zero, and thus, $\bbeta \to 1$. For some $\Delta \in (0,1)$, $\bbeta \triangleq 1/(1 + r\Delta)$ for the largest $r \in \iset{ L }$ such that $\Delta \leq (1-r/L)^M$. We can similarly study the decoder-only transformer with sparse/hard attention.
\paragraph{Efficient transformers.}
While transformers are able to handle arbitrary length input, their computational cost scales quadratically in the input length $L$ (compared to the linear cost of recurrent models). To mitigate this issue, various ``efficient'' transformers have been proposed~\cite{tay2022efficient,lin2022survey}. Various sparse attention mechanisms have been studied, utilizing block-local, dilated, global or banded attention (see Lin et al.\citeyear[Figure 4]{lin2022survey}). These architectures reduce the quadratic cost often to almost linear. In terms of their compositional complexity, these architectures can easily be studied within our proposed framework, and will have significantly lower complexity compared to the vanilla transformer (smaller $\bdelta$, similar $\bbeta$).
However, these architectures induce an input-agnostic sparsity pattern, thus leading to input-agnostic {\cdag}s. In contrast, sparse attention schemes with input-dependent sparsity patterns such as top-$K$ attention~\cite{gupta2021memory} and sparse Sinkhorn attention~\cite{tay2020sparse} induce input-dependent {\cdag}s.
\begin{table}[t]
\caption{Complexities of existing models.
LoI is specified approximately for the ease of exposition.
{\bf IDC:} Input-dependent \cdag. 
{\bf AL: } Arbitrary length operation.
$\blacklozenge$: The \cdag can be input-dependent if a input parse tree is available.
$\dag$: Conv+Pool induces input-dependent {\cdag}s for max/min-pool, not for avg/sum-pool.
$\ddag$: The number of sinks $m$ needs to be specified for conv+pool, and the model can handle arbitrary length if it can recursively conv+pool until the number of nodes is reduced to $m$.
$\bullet$: See discussion after \cref{prop:model:dotrf}.}
\label{tab:existing-models:skinny}
{
\begin{center}
\begin{tabular}{lcccll}
\toprule
Model & {\bf IDC}      & {\bf AL} & $(k,q,m)$ & $\bdelta$ & $\bbeta$ \\
\midrule
No-hier.(\ref{fig:models:flat})                & \xmark & \xmark & $(1,1,L)$               & $c$          & $1/L$                       \\
\midrule
U-RNN(\ref{fig:models:unirnn}) & \xmark & \cmark & $(2,1,1)$               & $c^{L-1}$   & $1/2$                        \\
B-RNN(\ref{fig:models:birnn})  & \xmark & \cmark & $(2,2,2)$               & $c^{L-1}$   & $1/4$                        \\
T-RNN(\ref{fig:models:bintree}) & $\blacklozenge$ & \cmark & $(2,1,1)$               & $c^{\log L}$& $1/L$                        \\
\midrule
Cnv+pl(\ref{fig:models:conv})          & $\dag$ & $\ddag$ & $(\text{$w$+$p$}, w, m)$ & $c^{\log L}$  & $\frac{2 / L}{1+\nicefrac{1}{p}}$ \\
\midrule
Trf(\ref{fig:models:trf})                    & \xmark & \cmark & $(L, L, 1)$           & $(Lc)^M$    & $1/L$                        \\
SA-Trf(\ref{fig:models:trf-hat}) & \cmark & \cmark & $(K, L, 1)$           & $L(Kc)^M$   & $1/K$                        \\
D-Trf(\ref{fig:models:dotrf}) ${}^\bullet$    & \xmark & \cmark & $(L, L, 1)$           & $(Lc)^M$    &  $\frac{1}{(1 + r\Delta)}$    \\
\bottomrule
\end{tabular}
\end{center}}
\end{table}

\subsection{Discussion}\label{sec:models:disc}

We discussed various existing sequence processing models in the context of our definition of compositional functions (\cref{def:comp-func,def:func-class}), and quantify their corresponding complexities in the form of bounds on the absolute LoI $\bdelta$ and relative LoI $\bbeta$.
We summarize these properties in \cref{tab:existing-models:skinny} to compare different models. Beyond complexities, we also specify (i)~whether they can operate on arbitrarily long sequences, and (ii)~whether the {\cdag}s are input-dependent or input-agnostic for some fixed input length $L$.
In \cref{asec:model:arb-len},
we discuss a model's ability to process arbitrary lengths, and how it relates to parameter sharing across different \cdag levels. 
Here, we focus on the topic of input-agnostic vs input-dependent {\cdag}s.

As we summarize in \cref{tab:existing-models:skinny}, the only model with explicitly input-dependent \cdag  is the sparse attention versions of the transformer (and the tree based recurrence if the tree is input-dependent).
The conv+pool composition implicitly induces an input-dependent \cdag if the pooling operation is explicitly selective (such as using max-pooling or min-pooling).

Given (i)~the success of the softmax attention mechanism in sequence processing tasks, (ii)~the use of explicitly input-dependent {\cdag}s in efficient transformers such as Sinkhorn transformers, tree-based RNNs, and specialized models for compositional generalization,
and (iii)~the wide use of max/min-pooling instead of sum/average pooling (especially in vision tasks), we think it is important to rigorously understand the value of input-dependent {\cdag}s. First, it is fair to assume that most realistic problems involve input-dependent {\cdag}s.
For example, if the input, and the corresponding output, are generated from a grammar, then the \cdag of any input $X$ would be closely related to its parse-tree.

To this end, we specifically study the ability of the compositional function with an input-agnostic \cdag to approximate a compositional function with input-dependent {\cdag}s. We present a condensed version of the result for ease of exposition;
the detailed version and proof is presented in \cref{asec:iddag-theory}:
\begin{theorem}[condensed] \label{thm:model:cdag}
Consider a $(k,q,m,\bdelta, \bbeta)$-compositional function class $\mathcal F$, and input sequences $X \in \mathcal X$ of length $L$. Consider a ground-truth function $f \in \mathcal F$ with components $e, D, g, h$, and a compositional function $\mathsf f \in \mathcal F$ with components $e, \mathsf D, \mathsf g, \mathsf h$, with an input-agnostic \cdag such that $\mathsf D(X) = {\mathsf{D}}\, \forall X \in \mathcal X, |X| = L$.
\footnote{We assume that the token encoder $e$ is same for both $f, \mathsf f$.}
Then the worst-case approximation of $f$ by $\mathsf f$, for $f, \mathsf f \in \mathcal F$ is given by:

{\begin{equation}\label{eq:iddag-theory}
C_l \bdelta \leq
\max_{\substack{
D, g, h, \\
f \triangleq \{e, D, g, h\},\\
f \in \mathcal F,
X \in \mathcal X
}}
\min_{\substack{
{\mathsf{D}}, \mathsf g, \mathsf h, \\
\mathsf f \triangleq \{e, \mathsf{D}, \mathsf g, \mathsf h \},\\
\mathsf f \in \mathcal F
}}
\left|
f(X) - \mathsf f (X)
\right|
\leq C_u 
 \frac{\bdelta}{\bbeta},
\end{equation}}%
where $f(X) = h(g^{\otimes D(X)}(e(x_1), \ldots, e(x_L))$, and $\mathsf f(X) = \mathsf h(\mathsf g^{\otimes {\mathsf{D}}}(e(x_1), \ldots, e(x_L))$, with general smoothness and structural assumptions, and universal constants $C_l, C_u > 0$.
\end{theorem}

This result provides an upper and lower bound on the approximation error. Even if we select the best possible fixed \cdag $\mathsf{D}$, and corresponding span processor $\mathsf g$ and read-out function $\mathsf h$ (the $\min$ over $\{{\mathsf{D}}, \mathsf g, \mathsf h\}$ for the approximation $\mathsf f \in \mathcal F$), there are compositional functions $f$ such that, for some input $X$ (the $\max$ over $\{D, g, h\}$ for $f \in \mathcal F$, and the input $X\in \mathcal X$), the approximation error is at least $\bigO(\bdelta)$.
The approximation is bounded from above by $\bigO(\bdelta/\bbeta)$; note that the maximum relative LoI $\bbeta \in [1/L, 1]$. This indicates that function classes with large $\bdelta$ are hard to approximate with an input-agnostic \cdag of the same complexity. Furthermore, for the same $\bdelta$, smaller $\bbeta$ worsen the upper bound.
Overall, function classes with small $\bdelta$, or moderate $\bdelta$ with large $\bbeta$ can be approximated well with input-agnostic {\cdag}s.  As discussed earlier, for function classes of particular interest to us, with moderately high $\bdelta$ and high $\bbeta$, input-agnostic {\cdag}s do not provide promising approximation, though larger $\bbeta$ is more favorable.
This result makes precise the intuition that input-agnostic {\cdag}s are not expressive enough to appropriately approximate functions with input-dependent {\cdag}s.

We also study systematic generalization within our framework. We consider a special class of compositional functions with a maximum out-degree $q=1$ in the \cdag, inducing ``cleanly-separable'' sub-structures, and further focus on the case with maximum in-degree $k=2$ and a single sink ($m=1$) in the \cdag. Given a class of (compositional) functions $\mathcal F$, and training data $S$, generalization guarantees bound the difference between the true risk $R(\hat f) = \mathbb E_{(X, y)} \ell(y, \hat f(X))$ of a learned model $\hat f \in \mathcal F$ and its empirical risk $R_N(\hat f)$. Here we present a condensed result;
see \cref{asec:sysgen-theory} for details.
\begin{theorem}[condensed] \label{thm:model:sysgen}
Consider a $(2,1,1,\bdelta, \bbeta)$-compositional function class $\mathcal F$, input $X \in \mathcal X$ of length $L$, a training set $S$ of $N$ samples from a ground-truth function $f \in \mathcal F$ with components $e, D, g, h$. Assume that the token encoder $e$ and the \cdag function $D$ are given, and we only learn the span encoder $\hat g$ and the read-out function $\hat h$ to get $\hat f \triangleq (e, D, \hat g, \hat h)$. Then, with probability at least $1 - \xi$,
{\begin{equation}
\left | R(\hat f) - R_N(\hat f) \right|
\leq \gamma \bdelta C_N \left (1 + 2N \sqrt{\frac{2 \log (2/\xi)}{N}} \right),
\end{equation}}
\noindent
where $C_N$ is a quantity dependent on $N$ and $\xi \in (0,1)$.
\end{theorem}
If the model class $\mathcal F$ is expressive, the empirical risk $R_N(\hat f)$ can be small.
If $C_N\sim O(\nicefrac{1}{N})$, then we recover the standard $O(\nicefrac{1}{\sqrt{N}})$ rate. However, if the compositional complexity, $\bdelta$, is high, this bound does not provide a favorable systematic generalization guarantee, highlighting that, even with cleanly-separable compositions ($q=1$), our proposed compositional complexity is directly tied to the systematic generalization.
\section{Conclusion}\label{sec:conc} 
In this paper, we proposed a precise novel definition of compositional functions, separating out the neural and symbolic parts, and a consequent compositional complexity, which explicitly quantifies the complexity in the (symbolic) computation trace or the \cdag of any compositional function.
We demonstrated how existing models, such as recurrent, convolutional or attention-based ones, fit into our proposed definition, allowing us to compute and compare their compositional complexities.
We categorized models into those with input-dependent compositions and those with input-agnostic ones, and established theoretical guarantees on the (in)ability of input-agnostic compositions to approximate input-dependent ones.
Based on this theoretical framework, we also establish compositional generalization guarantees for learned compositional functions, rigorously connecting our proposed notion of compositional complexity to systematic generalization.
\clearpage
\bibliographystyle{plainnat}
\bibliography{refs}

\clearpage
\appendix
\onecolumn

\section{Technical Details for \S \ref{sec:cdefs}} \label{asec:cdefs}
\subsection{Node ordering at each level of the \cdag} \label{asec:cdefs:total-ordering}
A \cdag $D(X) \triangleq \{N(X), E(X)\}$ for a specific input sequence $X \in \mathcal X$ can depend on $X$ or be pre-specified. This \cdag is a leveled DAG with set of nodes $N(X)$ and edges $E(X)$. Each node $n \triangleq (l:i) \in N(X)$ has a level $l$ and index $i$ 
--
the level $l$ of $n$ is one more than the highest level of any of its at most $k$ parents, the index $i$ of $n$ (at its level) is based on the sort order of the (level, index) tuples of its parents $P(n) = \{n' \triangleq (l':i') \in N(X): (n' \to n) \in E(X)\}$ (these tuples themselves are sorted based on their position in the argument list of $g$). This induces a total ordering at each level.

The cDAG function $D$ provides the argument order in $g$, which is then used to sort the nodes in the same level. There can be multiple nodes with same set of parents, but their ordering in the argument list must be different. We do not need two nodes with the same set of parents and same argument ordering (which leads to a partial ordering) since we can collapse such nodes into one, and collect their outgoing edges. Thus, we can always get a total ordering of the nodes at each level. If the $g$ function is independent of the ordering of its arguments, then we can just utilize the sorted order of the (level, index) tuples of the parents to get the total ordering.

\subsection{Multi-arity of span processors} \label{asec:cdefs:g-arity}
Our current definition states that the arity of the span processor $g: \mathcal H^k \to \mathcal H$ is fixed to $k$. We do so for the ease of exposition, though our definition can incorporate more flexible span processors.

For example, with a maximum in-degree of $k$ in the \cdag, we can consider a $g: \mathcal H^k \to \mathcal H$, and pad the input to $g$ for internal nodes with less than $k$ parents -- that is, for a node $n$ with $\tilde k < k$ parents $P(n)$ with values $\{v_{l_1:i_1}, \ldots, v_{l_{\tilde k}:i_{\tilde k}}\}$, the value of node $v_n = g(v_{l_1:i_1}, \ldots, v_{l_{\tilde k}:i_{\tilde k}}, \text{\tt [PAD]}, \ldots, \text{\tt [PAD]})$, where the padding is of length $(k-\tilde k)$.

Another possibility is to consider span processors $g: \cup_{j=2}^k \mathcal H^j \to \mathcal H$ that can handle multi-arity inputs. One example of such a function would be a weighted average over all the inputs (where the weights depend on the input). The in-degree $k$ of any node in the \cdag $D(X)$, and thus, the maximum arity of $g$ for the function $f$ under consideration, drives the length of the span needed to be considered to process this function. In some general purpose models, we will see that $k$ can be as large at $L$, the length of the input.

\subsection{Relationship between sensitivity and compositional complexity of a function}
\label{asec:cdefs:fs-1}

The \cref{def:loi} of the complexity of composition incorporates both the complexity of the \cdag $D(X)$ and the complexity of the span processor $g: \mathcal H^k \to \mathcal H$ in terms of its smoothness, with higher values of $c$ indicating more complex (less smooth) $g$. The absolute LoI $\delta_i$ incorporates the effect of longer paths, with the effect growing with path length, and corresponds to the sensitivity of the compositional function output to any one input token in the sequence. This sensitivity can be formalized as follows:
\begin{lemma}
\label{alem:fcomp-sensitivity}
Consider a compositional function $f: \mathcal X \to \mathcal Y$ with components $e, D, g, h$ (as in Definitions~\ref{def:comp-func}~\&~\ref{def:loi}).
Let $X = [x_1, \ldots, x_L]$ be an input of length $L$, and  $\Tilde{X}^j = [\Tilde{x}_1, \ldots, \Tilde{x}_L]$ be its ``perturbation'' at the $j^{\text{th}}$ index, with $x_i = \Tilde{x}_i \forall i \in \iset{ L }, i \not= j$, such that $D(X) = D(\Tilde{X}^j)$. Further, assume that $\mathcal Y = \mathbb R$, and let $\gamma > 0$ be the Lipschitz constant of $h: \mathcal H^m \to \mathbb R$. Then the sensitivity of $f$ with respect to the $j^{\text{th}}$ index is bounded as:
\begin{equation}\label{aeq:fcomp-sensitivity}
\left | f(X) - f(\Tilde{X}^j) \right |
\leq \gamma \delta_j \| e(x_j, j) - e(\Tilde{x}_j, j) \|.
\end{equation}
\end{lemma}

\begin{proof}
Let $S$ be the set of sink nodes in $D(X)$. Then $h$ maps the values $v_n$ of the sink nodes $n \in S$ to $\mathcal Y = \mathbb R$. Since $D(X) = D(\Tilde{X}^j)$, the set of sink nodes in $D(\Tilde{X}^j)$ is also $S$ but with different values $\Tilde{v}_n, n \in S$.

The perturbed input $\Tilde{X}^j$ is different from $X$ only at the index $j$, hence the source node values for $D(X)$ are the same for $D(\Tilde{X}^j)$. That is, $v_{\text{\tt 0}:i} =  \Tilde{v}_{\text{\tt 0}:i}$ for all $i \in \iset{ L }, i \not=j$.

Consider one sink node $n \in S$ in $D(X) = D(\Tilde{X}^j)$. Their values $v_n$ and $\Tilde{v}_{n}$ will be different only if there is a path to $n$ from {\tt 0}:$j$ (the perturbed source node). Let {\tt 0}:$j \xrightarrow{\text{path}} a$ denote that there is a path from the source node {\tt 0}:$j$ to a non-source node $a$.

Given the smoothness of $h$, we have the following:

\begin{equation} \label{aeq:hsmooth}
\left | h(v_{n_1}, \ldots, v_{n_m})
- h(\Tilde{v}_{n_1}, \ldots, \Tilde{v}_{n_m}) \right|
\leq \gamma \sum_{n \in S} \left \| v_n - \Tilde{v}_n \right \|
\mathbb I\left(\text{\tt 0}\text{:}j \xrightarrow{\text{path}} n \right).
\end{equation}

Consider one term corresponding to sink node $n \in S$ in the summation on the right-hand-side of \eqref{aeq:hsmooth}, and utilizing Definition~\ref{def:loi}, we can say that

\begin{align}
\left \| v_n - \Tilde{v}_n \right \|
\mathbb I\left(\text{\tt 0}\text{:}j \xrightarrow{\text{path}} n \right)
&
\leq \sum_{n_1 \in \mathcal P(n)} c \left\| v_{n_1} - \Tilde{v}_{n_1} \right\|
\mathbb I\left(\text{\tt 0}\text{:}j \xrightarrow{\text{path}} n_1 \right)
\\
&
\leq \sum_{n_1 \in \mathcal P(n)} c \left(
\sum_{n_2 \in \mathcal P(n_1)} c
\left\| v_{n_2} - \Tilde{v}_{n_2} \right\|
\mathbb I\left(\text{\tt 0}\text{:}j \xrightarrow{\text{path}} n_2 \right)
\right)
\\
\label{aeq:dag-trav}
&
\leq \sum_{n_1 \in \mathcal P(n)} c \left(
\sum_{n_2 \in \mathcal P(n_1)} c \left(
\cdots
c \left(
\sum_{n_l \in \mathcal P(n_{l-1})} c
\left\| v_{n_l} - \Tilde{v}_{n_l} \right\|
\mathbb I\left(\text{\tt 0}\text{:}j = n_l \right)
\right ) \right) \right),
\end{align}
where $v_a$ is the value of any node $a$, and $\Tilde{v}_a$ is the value of the same node with the perturbed input, $\mathcal P(a)$ is the set of parents of any node $a$. The right-hand-side of \eqref{aeq:dag-trav} is essentially considering the set $P_n(\text{\tt 0}\text{:}j)$ of all paths from source node {\tt 0}:$j$ to sink node $n$, and can be re-written as

\begin{align}
\left \| v_n - \Tilde{v}_n \right \|
\mathbb I\left(\text{\tt 0}\text{:}j \xrightarrow{\text{path}} n \right)
&
\leq
\sum_{\text{path } P \in P_n(\text{\tt 0}:j)}
c^{|P|}
\left\| v_{\text{0}:j} - \Tilde{v}_{\text{0}:j} \right\|
\\
&
=
\sum_{\text{path } P \in P_n(\text{\tt 0}:j)}
c^{|P|}
\left\| e(x_j, j) - e(\Tilde{x}_j, j) \right\|.
\end{align}
Substituting the above in \eqref{aeq:hsmooth} gives us

\begin{align}
\left | h(v_{n_1}, \ldots, v_{n_m})
- h(\Tilde{v}_{n_1}, \ldots, \Tilde{v}_{n_m}) \right|
& \leq \gamma \sum_{n \in S}
\sum_{\text{path } P \in P_n(\text{\tt 0}:j)}
c^{|P|}
\left\| e(x_j, j) - e(\Tilde{x}_j, j) \right\|
\\
& = \gamma
\underbrace{
\left (
\sum_{\text{path } P \in P(\text{\tt 0}\text{:}j)}
c^{|P|}
\right)
}_{\delta_j \text{ by Definition~\ref{def:loi}}}
\left\| e(x_j, j) - e(\Tilde{x}_j \right\|
\\
& = \gamma \delta_j
\left\| e(x_j, j) - e(\Tilde{x}_j, j) \right\|,
\end{align}
where $P(\text{\tt 0}\text{:}j)$ is the set of all paths from the source node {\tt 0}:$j$ to any sink node $n \in S$, giving us the statement of the lemma.
\end{proof}

The sensitivity of compositional functions $f$ in the compositional function class $\mathcal F$ (as defined in \cref{def:func-class}) to changes in the input can be characterized as follows with respect to the compositional complexity (as defined in \cref{def:loi}).

\begin{corollary}\label{acor:fcomp-sensitivity}
Given a class $\mathcal F$ of $(k,q,m, \bdelta, \bbeta)$-compositional functions as in Definition~\ref{def:func-class}, and the conditions and notations of Lemma~\ref{alem:fcomp-sensitivity}, $\forall f \in \mathcal F$ and any $X \in \mathcal X$ of length $L$, $\exists C > 0$ (universal constant) such that
\begin{equation}
\max_{j \in \iset{ L }}
\left| f(X) - f(\Tilde{X}^j) \right| \leq C \gamma \bdelta.
\end{equation}
\end{corollary}

\begin{proof}
Let $C = \max\limits_{\substack{x, x' \in \mathcal I,\\ i \in \iset{ L }}} \| e(x, i) - e(x', i) \|$. Then for any $X, \Tilde{X}^j \in \mathcal X$, by Lemma~\ref{alem:fcomp-sensitivity} and Definition~\ref{def:func-class}, we know that
\begin{align}
  \left| f(X) - f(\Tilde{X}^j) \right|
  & \leq \gamma \bdelta \| e(x_j, j) - e(\Tilde{x}_j, j) \| \\
  & \leq \gamma \bdelta C,
\end{align}
giving us the statement of the corollary.
\end{proof}

\clearpage
\section{Details for \S \ref{sec:models}}\label{asec:exp}

\begin{figure*}[t]
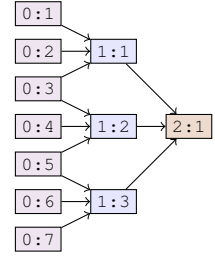
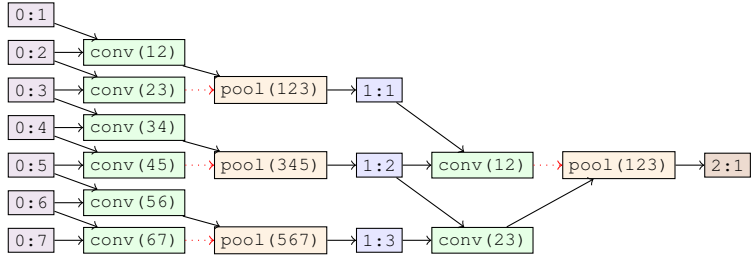
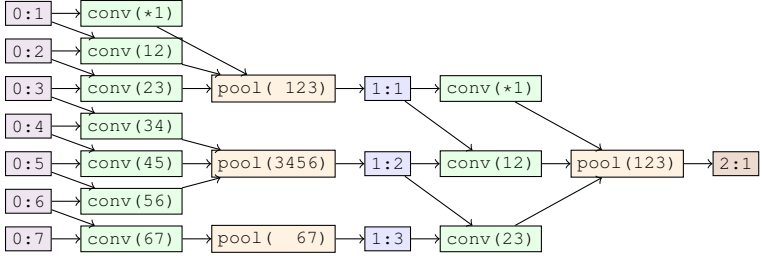
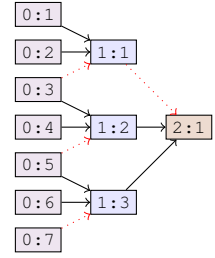

\centering

\begin{subfigure}{0.5\textwidth}
\centering
{\scriptsize {\tt \tikz
  \graph[
    nodes={draw, fill=blue!10},
    grow right sep=0.15in,
    group shift=(-90:0.5),
  ]
  {
    a[as=0:1, fill=Fuchsia!10];
    b[as=0:2, fill=Fuchsia!10]
    -> ab [as=conv(12),fill=green!10];
    a -> ab;
    c[as=0:3, fill=Fuchsia!10]
    -> bc[as=conv(23),fill=green!10]
    -> abc[as=pool(123),fill=BurntOrange!10]
    -> c1[as=1:1];
    ab -> abc;
    b -> bc;
    d[as=0:4, fill=Fuchsia!10]
    -> cd[as=conv(34),fill=green!10];
    c -> cd;
    e[as=0:5, fill=Fuchsia!10]
    -> de[as=conv(45),fill=green!10]
    -> cde[as=pool(345),fill=BurntOrange!10]
    -> c2[as=1:2]
    -> c12[as=conv(12),fill=green!10]
    -> c123[as=pool(123),fill=BurntOrange!10]
    -> fin[as=2:1, fill=Sepia!10];
    d -> de;
    c1 -> c12;
    cd -> cde;
    f[as=0:6, fill=Fuchsia!10]
    -> ef[as=conv(56),fill=green!10];
    e -> ef;
    g[as=0:7, fill=Fuchsia!10]
    -> fg[as=conv(67),fill=green!10]
    -> efg[as=pool(567),fill=BurntOrange!10]
    -> c3[as=1:3]
    -> c23[as=conv(23),fill=green!10];
    ef -> efg;
    c2 -> c23;
    c23 -> c123;
    f -> fg;
  };
}}
\caption{Convolution + Pooling (Detailed version of figure~\ref{fig:models:conv}): Convolution over a window of size 2 (no padding), and pooling over a window of size 2.}
\label{afig:models:conv}
\end{subfigure}
\hfill
\begin{subfigure}{0.2\textwidth}
\centering
{\scriptsize {\tt \tikz
  \graph[
    nodes={draw, fill=blue!10},
    grow right sep=0.15in,
    group shift=(-90:0.5),
  ]
  {
    a[as=0:1, fill=Fuchsia!10];
    b[as=0:2, fill=Fuchsia!10]
    -> bb [as=1:1];
    a -> bb;
    c[as=0:3, fill=Fuchsia!10];
    c -> bb;
    d[as=0:4, fill=Fuchsia!10]
    -> dd[as=1:2]
    -> ddd[as=2:1, fill=Sepia!10];
    c -> dd;
    bb -> ddd;
    e[as=0:5, fill=Fuchsia!10];
    e -> dd;
    f[as=0:6, fill=Fuchsia!10]
    -> ff[as=1:3];
    e -> ff;
    ff -> ddd;
    g[as=0:7, fill=Fuchsia!10];
    g -> ff;
  };
}}
\caption{Conv+Pool (same as Figure~\ref{fig:models:conv}}
\label{afig:models:conv-cnd}
\end{subfigure}
\begin{subfigure}{0.7\textwidth}
\centering
{\scriptsize {\tt \tikz
  \graph[
    nodes={draw, fill=blue!10},
    grow right sep=0.15in,
    group shift=(-90:0.5),
  ]
  {
    a[as=0:1, fill=Fuchsia!10];
    b[as=0:2, fill=Fuchsia!10]
    -> ab [as=conv(12),fill=green!10];
    a -> ab;
    c[as=0:3, fill=Fuchsia!10]
    -> bc[as=conv(23),fill=green!10]
    ->[red, dotted] abc[as=pool(123),fill=BurntOrange!10]
    -> c1[as=1:1];
    ab -> abc;
    b -> bc;
    d[as=0:4, fill=Fuchsia!10]
    -> cd[as=conv(34),fill=green!10];
    c -> cd;
    e[as=0:5, fill=Fuchsia!10]
    -> de[as=conv(45),fill=green!10]
    ->[red, dotted] cde[as=pool(345),fill=BurntOrange!10]
    -> c2[as=1:2]
    -> c12[as=conv(12),fill=green!10]
    ->[red, dotted] c123[as=pool(123),fill=BurntOrange!10]
    -> fin[as=2:1, fill=Sepia!10];
    d -> de;
    c1 -> c12;
    cd -> cde;
    f[as=0:6, fill=Fuchsia!10]
    -> ef[as=conv(56),fill=green!10];
    e -> ef;
    g[as=0:7, fill=Fuchsia!10]
    -> fg[as=conv(67),fill=green!10]
    ->[red, dotted] efg[as=pool(567),fill=BurntOrange!10]
    -> c3[as=1:3]
    -> c23[as=conv(23),fill=green!10];
    ef -> efg;
    c2 -> c23;
    c23 -> c123;
    f -> fg;
  };
}}
\caption{Conv+Pool: Detailed version of selective pooling with Figure~\ref{afig:models:conv}.}
\label{afig:models:conv-p1}
\end{subfigure}
\hfill
\begin{subfigure}{0.2\textwidth}
\centering
{\scriptsize {\tt \tikz
  \graph[
    nodes={draw, fill=blue!10},
    grow right sep=0.15in,
    group shift=(-90:0.5),
  ]
  {
    a[as=0:1, fill=Fuchsia!10];
    b[as=0:2, fill=Fuchsia!10]
    -> bb [as=1:1];
    a -> bb;
    c[as=0:3, fill=Fuchsia!10];
    c ->[red,dotted] bb;
    d[as=0:4, fill=Fuchsia!10]
    -> dd[as=1:2]
    -> ddd[as=2:1, fill=Sepia!10];
    c -> dd;
    bb ->[red,dotted] ddd;
    e[as=0:5, fill=Fuchsia!10];
    e ->[red,dotted] dd;
    f[as=0:6, fill=Fuchsia!10]
    -> ff[as=1:3];
    e -> ff;
    ff -> ddd;
    g[as=0:7, fill=Fuchsia!10];
    g ->[red,dotted] ff;
  };
}}
\caption{Fig~\ref{afig:models:conv-p1} condensed}
\label{afig:models:conv-p1-cnd}
\end{subfigure}

\begin{subfigure}{0.7\textwidth}
\centering
{\scriptsize {\tt \tikz
  \graph[
    nodes={draw, fill=blue!10},
    grow right sep=0.15in,
    group shift=(-90:0.5),
  ]
  {
    a[as=0:1, fill=Fuchsia!10]
    -> aa [as=conv(*1),fill=green!10];
    b[as=0:2, fill=Fuchsia!10]
    -> ab [as=conv(12),fill=green!10];
    a -> ab;
    c[as=0:3, fill=Fuchsia!10]
    -> bc[as=conv(23),fill=green!10]
    -> abc[as=pool(\ 123),fill=BurntOrange!10]
    -> c1[as=1:1]
    -> cc1 [as=conv(*1),fill=green!10];
    ab -> abc;
    aa -> abc;
    b -> bc;
    d[as=0:4, fill=Fuchsia!10]
    -> cd[as=conv(34),fill=green!10];
    c -> cd;
    e[as=0:5, fill=Fuchsia!10]
    -> de[as=conv(45),fill=green!10]
    -> cdef[as=pool(3456),fill=BurntOrange!10]
    -> c2[as=1:2]
    -> c12[as=conv(12),fill=green!10]
    -> c123[as=pool(123),fill=BurntOrange!10]
    -> fin[as=2:1, fill=Sepia!10];
    d -> de;
    c1 -> c12;
    cc1 -> c123;
    cd -> cdef;
    f[as=0:6, fill=Fuchsia!10]
    -> ef[as=conv(56),fill=green!10];
    e -> ef;
    g[as=0:7, fill=Fuchsia!10]
    -> fg[as=conv(67),fill=green!10]
    -> fg1[as=pool(\ \ 67),fill=BurntOrange!10]
    -> c3[as=1:3]
    -> c23[as=conv(23),fill=green!10];
    ef -> cdef;
    c2 -> c23;
    c23 -> c123;
    f -> fg;
  };
}}
\caption{Convolution + Pooling (Detailed version): Convolution over a window of size 2 (with padding), and pooling over a window of size 3.}
\label{afig:models:conv2}
\end{subfigure}
\hfill
\begin{subfigure}{0.2\textwidth}
\centering
{\scriptsize {\tt \tikz
  \graph[
    nodes={draw, fill=blue!10},
    grow right sep=0.15in,
    group shift=(-90:0.5),
  ]
  {
    a[as=0:1, fill=Fuchsia!10];
    b[as=0:2, fill=Fuchsia!10]
    -> bb [as=1:1];
    a -> bb;
    c[as=0:3, fill=Fuchsia!10];
    c -> bb;
    d[as=0:4, fill=Fuchsia!10]
    -> dd[as=1:2]
    -> ddd[as=2:1, fill=Sepia!10];
    c -> dd;
    bb -> ddd;
    e[as=0:5, fill=Fuchsia!10];
    e -> dd;
    f[as=0:6, fill=Fuchsia!10]
    -> ff[as=1:3];
    f -> dd;
    ff -> ddd;
    g[as=0:7, fill=Fuchsia!10];
    g -> ff;
  };
}}
\caption{Condensed version of Figure~\ref{afig:models:conv}}
\label{afig:models:conv2-cnd}
\end{subfigure}

\caption{Existing models - Extra details}
\label{afig:models:extra}
\end{figure*}

\subsection{Proof of Proposition~\ref{prop:model:unirnn}} \label{asec:exp:unirnn}

\begin{proposition}\label{aprop:model:unirnn}
With unidirectional recurrent composition, the maximum absolute LoI is $\bdelta \triangleq c^{L-1}$, with a maximum relative LoI of $\bbeta \triangleq (c^L - c^{L-1})/(2c^L -c^{L-1} - 1)$.
\end{proposition}
\begin{proof}
Both source nodes {\tt 0}:$1$ and {\tt 0}:$2$ have a single path to the single sink node with path length $L-1$. Thus $\delta_1 = \delta_2 = c^{L-1}$. All source nodes {\tt 0}:$i$, $i \in [3,L]$, have a single path to the single single node with a path length of $L-i+1$. Thus $\sum_{i \in \iset{ L }} \delta_i = (2 c^L - c^{L-1} - 1) / (c-1)$.

Thus, $\bdelta = \delta_1 = \delta_2 = c^{L-1}$, with $\bbeta = \bdelta / \sum_{i \in \iset{ L }} \delta_i =  (c^L - c^{L-1})/(2c^L -c^{L-1} - 1)$.
\end{proof}

\subsection{Proof of Proposition~\ref{prop:model:birnn}} \label{asec:exp:birnn}

\begin{proposition}\label{aprop:model:birnn}
With bidirectional recurrent composition, the maximum absolute LoI is $\bdelta \triangleq c^{L-1} + c$, with a maximum relative LoI of $\bbeta \triangleq (c^L - c^{L-1} + c^2 - c)/2(2c^L -c^{L-1} - 1)$.
\end{proposition}
\begin{proof}
Source node {\tt 0}:$1$ has two paths to each of two sink nodes, with path lengths $L-1$  and $1$, which source node {\tt 0}:$L$ has two paths of lenghts $1$ and $L-1$. Thus $\delta_1 = \delta_L = c^{L-1}+c$. All source nodes {\tt 0}:$l$, $l \in [2,L]$, have two  paths to the two sink nodes with paths of lengths of $L-l+1$ and $l$. Thus $\sum_{i \in \iset{ L }} \delta_i = 2(2 c^L - c^{L-1} - 1) / (c-1)$.

Thus, $\bdelta = \delta_1 = \delta_2 = \delta_{L-1} = \delta_L = c^{L-1}+c$, with $\bbeta = \bdelta / \sum_{i \in \iset{ L }} \delta_i =  (c^L - c^{L-1} + c^2 - c)/2(2c^L -c^{L-1} - 1)$.
\end{proof}

\subsection{Proof of Proposition~\ref{prop:model:treernn}} \label{asec:exp:treernn}

\begin{proposition}\label{aprop:model:treernn}
With balanced binary-tree recurrent composition, and $L = 2^l$ for some positive integer $l$, the maximum absolute LoI is $\bdelta \triangleq c^{\log_2 L}$, with a maximum relative LoI of $\bbeta \triangleq 1/L$.
\end{proposition}
\begin{proof}
All source nodes {\tt 0}:$i$, $i \in \iset{ L }$, have a single path to the single sink node with path length $\log_2 L$. Thus $\bdelta \triangleq c^{ \log_2 L}$, and $\sum_{i \in \iset{ L }} \delta_i = L c^{\log_2 L}$, and $\bbeta = 1/L$.
\end{proof}

\subsection{Proof of Proposition~\ref{prop:model:convpool}} \label{asec:exp:convpool}

\begin{proposition}\label{aprop:model:convpool}
Assuming $1 < w, p \ll L$, the conv+pool composition has a maximum absolute LoI of $\bdelta \sim \bigO(c^{\log L})$, and a maximum relative LoI of $\bbeta \sim \bigO(\nicefrac{2}{L (1 + 1/p)})$.
\end{proposition}
\begin{proof}
We operate under the assumption that convolution does not reduce the length of the sequence (usually it does not increase, but it can reduce in case there is no padding at the ends of the sequence), and that pooling reduces the length of the input sequence. Pooling over a window of size $p$ reduces a sequence of length $L$ to $L/p$. This reduction in length corresponds to the number of nodes in the cDAG at each level. At level 0, there are $L$ source nodes, and after the conv+pool operation, there would be $L/p$ nodes at the level 1. If the user specified number of sink nodes is $m$, then, the maximum length of a source node to sink path would be $\log (L/m) / \log p \sim O(\log L)$.

We consider two cases: (i)~$w  > p$, (ii)~$p < w$.

In case~(i), $w > p$, so each node can be the parent of $(1+w/p)$ nodes in the next level, thereby each source node can have at most $\bigO((1+w/p)^{\log L})$ paths to a sink nodes. Thus $\bdelta \sim \bigO((1+w/p)^{\log L} c^{\log L}) \sim O(c^{\log L})$, while $\bbeta$ is close to $1/L$.

In case~(ii), $p > w$, so only $L/p$ nodes can have $>1$ paths to any sink nodes; the rest of the source nodes will have a single path to any sink node. In this case, $\bdelta$ is at most $2 C c^{\log L}$ for some constant $C$, and $\beta_i$ for any such node is given by

\begin{equation}
\beta_i = \frac{2 C c^{\log L}}{(2L/p + (L - L/p)) C c^{\log L}} = \frac{2}{L(1 - 1/p)}.
\end{equation}

Putting case (i) and (ii) together, we get $\bdelta \sim \bigO(c^{\log L})$ and $\bbeta = 2 / (L(1 + 1/p))$.
\end{proof}

\subsection{Transformer block based span processor} \label{asec:model:trf-g}

For a node $l+1$:$i$, the value is given by 
$$
v_{l+1:i} \gets g(v_{l:i}, v_{l:1}, \ldots, v_{l:L}) = \text{\sf LN}(\tilde v_{l:i} + W_2^{l} \cdot \text{\sf ReLU}(W_1^{l} \tilde v_{l:i})),
$$
and 
$$
\tilde v_{l:i} \gets \text{\sf LN}(v_{l:i} + \sum_{h \in \iset{ H } } W_h^l \sum_{j \in \iset{ L }} a(v_{l:i}Q_h^l, v_{l:j} K_h^l) V^l_h v_{l:i}).
$$ 
Here $\text{\sf LN}(\cdot)$ is layer-normalization, $a(\cdot, \cdot)$ is the normalized attention weight, $W_h^l, Q_h^l, K_h^l, V_h^l$ are the per-attention-head weights at level $l$.

\subsection{Proof of Proposition~\ref{prop:model:trf}} \label{asec:exp:trf}

\begin{proposition}\label{aprop:model:trf}
A transformer based composition with $M$ blocks has a the maximum absolute LoI of $\bdelta = {L}^{M+1} c^{M+1}$, and a maximum relative LoI of $\bbeta = 1/L$.
\end{proposition}
\begin{proof}
All source nodes {\tt 0}:$i$, $i \in \iset{ L }$, have a $L^M$ paths to the $j$-th node $M$:$j$, $j \in \iset{ L }$ at the $M$-th level, and each node in the $M$-th level has a single path to the sink node.

This gives a total number of paths of $L\times L^M = L^{M+1}$. Each path is exactly of length $M+1$.

Thus $\bdelta \triangleq L^{M+1} c^{M+1}$, and $\sum_{i \in \iset{ L }} \delta_i = L^{M+2} c^{M+1}$, and $\bbeta = 1/L$.
\end{proof}

\subsection{Proof of Proposition~\ref{prop:model:trf-att}} \label{asec:exp:trf-att}

\begin{proposition}\label{aprop:model:trf-att}
A transformer with $M$ blocks and $K$-sparse attention ($K\ll L$) has a maximum absolute LoI of $\bdelta = L K^{M}c^{M+1}$, and a maximum relative LoI of $\bbeta = 1/K$.
\end{proposition}
\begin{proof}
In the presence of sparse or hard attention, there are exactly $L(K+1)$ edges between nodes of two consecutive levels. The maximal number of paths for any source node {\tt 0}:$i^\star$ is in the scenario where, at each level $l$, there is a set of exactly $K$ nodes, $\mu_l$, which are in the top-$K$ for all nodes in level $l+1$, and the source node {\tt 0}:$i^\star$ is the parent of all the nodes in $\mu_1$. Note that, for any $b$ consecutive levels $l, l+1, \ldots, l+b$, $\mu_l, \mu_{l+1}, \ldots, \mu_b$ are effectively a $b$ fully connected layers.

Let $\mu_0$ be the set of source nodes which are in the top-$K$ of all nodes in level 1. In this scenario, the number of paths from a source node {\tt 0}:$i^\star \in \mu_0$ to the sink node involving only the nodes in the $\mu_l, l \in \iset{M}$ is $K^{M} \times K$, where the $M$ effectively fully connected layers produce the $K^{M}$ paths from {\tt 0}:$i^\star$ to any node in $\mu_{M}$, and every node in $\mu_{M}$ connects to the sink node, giving the $K$ term.

For the paths from a source node {\tt 0}:$i^\star$ in $\mu_0$ through the set of nodes $\bar \mu_l$ that are not in $\mu_l$ at level $l$, the number of paths is as follows:
\begin{itemize}
\item The number of paths from {\tt 0}:$i^\star$ to a node in $\bar \mu_1$ is 1.
\item The number of paths from {\tt 0}:$i^\star$ to a node in $\bar \mu_2$ is $K+1$, where $K$ paths are through the nodes in $\mu_1$, and the 1 additional path is through the node in $\bar \mu_1$.
\item The number of paths from {\tt 0}:$i^\star$ to a node in $\bar \mu_3$ is $K^2 + K +1$, where $K^2$ paths are through the nodes in $\mu_1$ and $\mu_2$, and the $K+1$ additional path is through the node in $\bar \mu_2$.
\item By induction, we can show that the number of paths from {\tt 0}:$i^\star$ to a node in $\bar \mu_l$ is $\sum_{j = 1}^l K^{j-1} = \nicefrac{(K^l - 1)}{(K -1)}$.
\end{itemize}

Thus, {\tt 0}:$i^\star$ has $\nicefrac{(K^M-1)}{(K-1)}$ paths to a node in $\bar \mu_M$, and there are $(L-K)$ nodes in $\bar \mu_M$, giving $(L-K) \nicefrac{(K^M-1)}{(K-1)}$ paths involving nodes not in $\mu_l, l \in \iset{ M }$.

This gives us a total number of paths 
\begin{equation*}
K \cdot K^M + (L-K) \cdot \frac{K^M - 1}{K-1} \leq L K^M,
\end{equation*}
thus, leading to a $\bdelta = L K^M c^{M+1}$ since each of the paths are of length exactly $M+1$.

This scenario also produces the maximal $\beta_{i^\star}$. The $K$ source nodes in $\mu_0$ all have a $\delta_i = L K^M c^{M+1}$. All other source nodes have $\bigO(1)$ paths to the sink node. In this case,
\begin{equation}
  \beta_{i^\star} = \frac{\delta_{i^\star}}{\sum_{i \in \iset{ L }} \delta_i}
  \leq \frac{\delta_{i^\star}}{\sum_{i \in \mu_0} \delta_i} = 1/K = \bbeta,
\end{equation}
thus giving us the maximal relative LoI.
\end{proof}

\subsection{Proof of Proposition~\ref{prop:model:dotrf}} \label{asec:exp:dotrf}

\begin{proposition}\label{aprop:model:dotrf}
  For a decoder-only transformer based composition with an input length $L$ and $M$ levels, the maximum absolute LoI is $\max_{i \in \iset{ L }} \delta_i = \delta_1 \triangleq \bdelta \sim \bigO( L^M c^{M+1})$, and a maximum relative LoI of $\bbeta \sim \bigO\left( L^M / \sum_{i \in \iset{ L }} i^M \right) \sim \bigO \left( 1 / \left (1 + \sum_{i \in \iset{ L-1 }} (i/L)^M \right) \right)$.
\end{proposition}

\begin{proof}
Given the structure of the decoder-only transformer, the source node {\tt 0}:{\tt 1} has the largest number of paths to the sink node. We will count those paths as follows:
\begin{itemize}
\item For any node {\tt 1}:$l$, $l \in \iset{ L }$ at level {\tt 1}, the number of paths from {\tt 0}:{\tt 1} to it is exactly 1.
\item For any node {\tt 2}:$l$, $l \in \iset{ L }$ at level {\tt 2}, the number of paths from {\tt 0}:{\tt 1} to it is 1 (through {\tt 1}:$l$) in addition to the number of paths from {\tt 0}:{\tt 1} to a node {\tt 1}:$j$ at level {\tt 1} for all $j < l$. From above, we know that there is only 1 path from {\tt 0}:{\tt 1}$\to${\tt 1}:$j$ for any $j$. Thus, there are exactly $l$ paths from {\tt 0}:{\tt 1}$\to${\tt 2}:$l$.
\item For any node {\tt 3}:$l$, $l \in \iset{ L }$ at level {\tt 3}, the number of paths from {\tt 0}:{\tt 1} to it is $l$ (through {\tt 2}:$l$) in addition to the number of paths from {\tt 0}:{\tt 1}$\to${\tt 2}:$j$ at level {\tt 2} for all $j < l$. From above, we know that there will be $j$ paths from {\tt 0}:{\tt 1}$\to${\tt 2}:$j$ for any $j$. Thus, there are exactly $\sum_{j \in \iset{ l }} j = l(l+1)/2$ paths from {\tt 0}:{\tt 1}$\to${\tt 3}:$l$.
\item For any node {\tt 4}:$l$, $l \in \iset{ L }$ at level {\tt 4}, the number of paths from {\tt 0}:{\tt 1} to it is $l(l+1)/2$ (through {\tt 3}:$l$) in addition to the number of paths from {\tt 0}:{\tt 1}$\to${\tt 3}:$j$ at level {\tt 3} for all $j < l$. From above, we know that there will be $j(j+1)/2$ paths from {\tt 0}:{\tt 1}$\to${\tt 3}:$j$ for any $j$. Thus, there are exactly $\nicefrac{1}{2}\sum_{j \in \iset{ l }} j(j+1)$ paths from {\tt 0}:{\tt 1}$\to${\tt 3}:$l$. Now $ \sum_{j \in \iset{ l }} j(j+1) \sim \bigO(\sum_{j \in \iset{ l }} j^2) \sim \bigO(l^3)$.
\item Continuing through this via induction, for any node $t$:$l$, $l \in \iset{ L }$ at level $t \leq M$, the number of paths from {\tt 0}:{\tt 1} to it is $\bigO(l^{t-2})$ (through $(t-1)$:$l$) in addition to the number of paths from {\tt 0}:{\tt 1}$\to$ $(t-1)$:$j$ at level $(t-1)$ for all $j < l$. From above, we know that there will be $\bigO(j^{t-2})$ paths from {\tt 0}:{\tt 1}$\to$ $(t-1)$:$j$ for any $j$. Thus, there will be $\bigO \left(\sum_{j \in \iset{ l }} j^{t-2} \right) \sim \bigO(l^{t-1})$ paths from {\tt 0}:{\tt 1}$\to$ $t$:$l$. 
\item Finally, following this induction, the number of paths from {\tt 0}:{\tt 1}$\to$ $M$:$l$ for $l \in \iset{ L }$ at the final $M$ level (after the $M$-th transformer block) will be $\bigO(l^{M-1})$.
\end{itemize}

Now summing the paths from {\tt 0}:{\tt 1} across all nodes $M$:$l$ at the $M$-th level, we get a total number of paths of $\bigO\left( \sum_{l \in \iset{L}} l^{M-1} \right) \sim \bigO(L^M)$. 

Each of these paths will be of length exactly $M+1$. Thus, the maximum LoI $\bdelta \sim \bigO(L^M c^{M+1})$.

Using the same process as above, we can show that the absolute LoI $\delta_i$ of the $i$-th source node {\tt 0}:$i$ for $i \in \iset{ L }$ will be given as $\delta_i \sim \bigO ( (L-i+1)^M )$. Thus 
\begin{equation}
\sum_{i \in \iset{ L } } \delta_i 
\sim
\bigO\left( \sum_{i \in \iset{ L }} (L-i+1)^M \right)
\sim
\bigO\left( \sum_{i \in \iset{ L }} i^M \right).
\end{equation}

Substituting this into the definition of $\bbeta$ gives us the order of magnitude of $\bbeta$ in the statement of the proposition.
\end{proof}

\subsection{Operating on arbitrary length sequences (and how that relates to parameter sharing).}
\label{asec:model:arb-len}
For the models we have discussed (\cref{tab:existing-models}, ignoring the no-hierarchical compositional model), the recurrent composition, the conv+pool composition, and the transformer-based composition can operate on arbitrary length input sequences, although the quadratic computation complexity (in the input sequence length $L$) often requires the model to put limits on the input sequence length for practical purposes. The multi-layered compositional models require the maximum sequence length to be specified beforehand, and the models will not operate on sequences of larger length -- the maximum length can be set a large value and suffice for all practical purposes, but of course this results in models with large number of parameters.
For the recurrent and conv+pool compositions, the maximum number of levels in the \cdag is proportional to the length $L$ of the input sequence. In this case, effectively the span processor $g$ is shared across all levels $l$ of the {\cdag} -- we are sharing parameters across all levels of the \cdag. This sharing can enable the ability to successfully operate on arbitrarily long sequences -- if a length was not previously encountered (say, during learning), the corresponding higher levels of the \cdag were not encountered either, but the model just applies the $g$ to all the levels and hence is operable. As previously discussed, the span processor $g$ can be made level dependent by including the level information of the parent nodes in the \cdag in the latent representation of any internal node either explicitly (for example, by appending the levels of the parents to the representation and defining $g: (\mathcal H \times \mathbb N)^k \to \mathcal H$ instead of the current $g: \mathcal H^k \to \mathcal H$) or implicitly (as is done in various recurrent models).
For compositions with a pre-specified number of levels, such as in the multi-layered fully-connected or the transformer-based compositions, the maximum number of levels $M$ in the \cdag are not directly related to the length $L$ of the sequence (and often is a hyperparameter). In this case, usually the span processor $g$ is level-dependent, explicitly being different span processor $g_l, l \in \iset{ M }$. We can still incorporate it in our definition which makes use of a single span processor $g$ by defining $g$ as a dictionary of functions $g = \{1:g_1, 2:g_2, \cdots, M:g_M\}$, and the level $l$ of a node $l$:$i$ being used to select the appropriate span processor $g[l]$. From the perspective of expressiveness, it might be beneficial to have distinct span processors for each level. However, this can be challenging during learning, and it makes sense to utilize weight-sharing across the levels as in the universal transformer~\cite{dehghani2019universal}. In fact, Csord{\'a}s {\em et al.}~\citeyear{csordas2021devil} highlight the ability of the universal transformer to compositionally generalize better than the general transformer because of this weight sharing.

\vfill

\begin{table*}[!hb]
\caption{Existing models in the context of Definition~\ref{def:func-class}. The LoI are specified approximately for the ease of exposition.
}
\label{tab:existing-models}
{\small
\begin{center}
\begin{tabular}{lcccll}
\toprule
Model & Input-dependent \cdag     & Arbitrary length & $(k,q,m)$ & $\bdelta$ & $\bbeta$ \\
\midrule
No-hierarchy (\ref{fig:models:flat})                & \xmark & \xmark & $(1,1,L)$               & $c$          & $1/L$                       \\
\midrule
Unidirectional recurrence (\ref{fig:models:unirnn}) & \xmark & \cmark & $(2,1,1)$               & $c^{L-1}$   & $1/2$                        \\
Bidirectional recurrence (\ref{fig:models:birnn})   & \xmark & \cmark & $(2,2,2)$               & $c^{L-1}$   & $1/4$                        \\
Tree recurrence (\ref{fig:models:bintree}) & $\blacklozenge$ & \cmark & $(2,1,1)$               & $c^{\log L}$& $1/L$                       \\
\midrule
Convolution then pooling (\ref{fig:models:conv})          & $\dag$ & $\ddag$ & $(\text{$w$+$p$}, w, m)$ & $c^{\log L}$  & $\frac{2}{L(1+\frac{1}{p})}$ \\
\midrule
Transformer (\ref{fig:models:trf})                    & \xmark & \cmark & $(L, L, 1)$           & $L^{M+1} c^{M+1}$    & $1/L$                        \\
Transformer + sparse attention (\ref{fig:models:trf-hat}) & \cmark & \cmark & $(K, L, 1)$           & $L K^M c^{M+1}$   & $1/K$                        \\
Decoder-only transformer (\ref{fig:models:dotrf})     & \xmark & \cmark & $(L, L, 1)$           & $L^M c^{M+1}$    &  $\nicefrac{1}{(1 + r\Delta)}^\bullet$    \\
\bottomrule
\end{tabular}
\end{center}}
$\blacklozenge$ The \cdag can be input-dependent if a parse tree for the input is available for the problem.\\
$\blacksquare$ The \cdag can be input-dependent if some form of hard gating or attention is used in the memory access and update.\\
$\dag$ Conv+Pool induces input-dependent {\cdag}s for max/min-pool, but not for avg/sum-pool.\\
$\ddag$ The number of sink nodes $m$ needs to be specified for conv+pool, and the model can handle arbitrary length if it is allowed to recursively conv+pool until the number of nodes is reduced to $m$.\\
$\bullet$ See discussion after \cref{prop:model:dotrf}.
\end{table*}

\clearpage
\section{Technical Details for Theorem~\ref{thm:model:cdag}}
\label{asec:iddag-theory}

\begin{theorem} \label{athm:model:cdag}
Consider a $(k,q,m,\bdelta, \bbeta)$-compositional function class $\mathcal F$, and input sequences $X \in \mathcal X$ of length $L$. Consider a ground-truth function $f \in \mathcal F$ with components $e, D, g, h$, and a compositional function $\mathsf f \in \mathcal F$ with components $e, \mathsf D, \mathsf g, \mathsf h$, with an input-agnostic \cdag such that $\mathsf D(X) = \mathsf D \, \forall X \in \mathcal X, |X| = L$.
\footnote{We are assuming that the token encoder $e$ is same for both $f, \mathsf f$.}
We further assume that, given any input-agnostic \cdag $\mathsf D$ and any $\mathsf g, \mathsf h$, $\exists X \in \mathcal X$ and $D, g, h$, such that (A1)~$D(X) \not= \mathsf D$, but is isomorphic to it, (A2)~$h(g^{\otimes \mathsf D}(X)) = \mathsf h (\mathsf g^{\otimes \mathsf D}(X))$, and (A3)~$|f(X) - f(\Tilde{X}^{i^\star})| = \kappa \delta_{i^\star} \| e_{i^\star} - \Tilde{e}_{i^\star} \|$, where $i^\star = \arg \max_{i \in \iset{ L }} \beta_i$ for $f:={e, D, g, h}$, $\Tilde{X}^{i^\star}$ is the perturbation of $X$ at the $i^\star$-th index, and $e_{i^\star} = e(x_{i^\star}, i^\star)$, $\Tilde{e}_{i^\star} = e(\Tilde{x}_{i^\star}, i^\star)$.
Then the worst-case approximation of $f$ by $\mathsf f$, for $f, \mathsf f \in \mathcal F$ is given by:

{\begin{equation*}\label{aeq:iddag-theory}
C_l \bdelta \leq
\max_{\substack{
D, g, h, \\
f := \{e, D, g, h\},\\
f \in \mathcal F,\\
X \in \mathcal X
}}
\min_{\substack{
\mathsf D, \mathsf g, \mathsf h, \\
\mathsf f := \{e, \mathsf D, \mathsf g, \mathsf h \},\\
\mathsf f \in \mathcal F
}}
\left|
f(X) - \mathsf f (X)
\right|
\leq C_u \frac{\bdelta}{\bbeta},
\end{equation*}}%
where $f(X) = h(g^{\otimes D(X)}(e(x_1), \ldots, e(x_L))$, and $\mathsf f(X) = \mathsf h(\mathsf g^{\otimes \mathsf D}(e(x_1), \ldots, e(x_L))$, with general smoothness and structural assumptions, and universal constants $C_l, C_u > 0$.
\end{theorem}

\begin{proof}

Focusing on the upper bound, we can that

\begin{align}
&
\max_{\substack{
D, g, h, \\
f := \{e, D, g, h\},\\
f \in \mathcal F,\\
X \in \mathcal X
}}
\min_{\substack{
    \mathsf D, \mathsf g, \mathsf h, \\
\mathsf f := \{e, \mathsf D, \mathsf g, \mathsf h \},\\
\mathsf f \in \mathcal F
}}
\left|
f(X) - \mathsf f (X)
\right|
\\
&
=
\max_{\substack{
D, g, h, \\
f := \{e, D, g, h\},\\
f \in \mathcal F,\\
X \in \mathcal X
}}
\min_{\substack{
\mathsf D, \mathsf g, \mathsf h, \\
\mathsf f := \{e, \mathsf D, \mathsf g, \mathsf h \},\\
\mathsf f \in \mathcal F
}}
\left|
h(g^{\otimes D(X)}(e(x_1, 1), \ldots, e(x_L, L))
- \mathsf h(\mathsf g^{\otimes \mathsf D}(e(x_1, 1), \ldots, e(x_L, L))
\right|
\\
&
\leq
\max_{\substack{
D, g, h, \\
f := \{e, D, g, h\},\\
f \in \mathcal F,\\
X \in \mathcal X
}}
\min_{\substack{
\mathsf D \\
\mathsf f := \{e, \mathsf D, g, h \},\\
\mathsf f \in \mathcal F
}}
\left|
h(g^{\otimes D(X)}(e(x_1, 1), \ldots, e(x_L, L))
- h(g^{\otimes \mathsf D}(e(x_1, 1), \ldots, e(x_L, L))
\right|
\label{aeq:ub-1}
\end{align}

Now consider the sequence $E = [e_1, \ldots, e_L]$ where $e_i = e(x_i, i)$. Given that we are maximizing over $D$ and $X$, we can select them such that, $D(X)$ is isomorphic to $\mathsf D$ but $D(X) \not= \mathsf D$. This is possible since both $D(X)$ and $\mathsf D$ have the same compositional complexity.

Now, let us sort the set of source nodes of $\mathsf D$ and of $D(X)$ with respect to their corresponding LoI $\delta_i$ to get the sorted indices $\Tilde{S}$ and $S$. Then, consider a reordering $\Tilde{E} = [\Tilde{e}_1, \ldots, \Tilde{e}_L ]$ of the sequence $E$, where $\forall i \in \iset{ L }$, $\Tilde{e}_{S[i]} \gets e_{\Tilde{S}[i]}$. Given the selected isomorphic $D(X)$, the following holds:

\begin{equation} \label{aeq:reorder-eq}
  h(g^{\otimes D(X)}(\Tilde{e}_1, \ldots, \Tilde{e}_L))
  = h(g^{\otimes \mathsf D}(e_1, \ldots, e_L)).
\end{equation}

Then substituting \eqref{aeq:reorder-eq} in \eqref{aeq:ub-1}, we have

\begin{align}
&
\max_{\substack{
D, g, h, \\
f := \{e, D, g, h\},\\
f \in \mathcal F,\\
X \in \mathcal X
}}
\min_{\substack{
    \mathsf D, \mathsf g, \mathsf h, \\
\mathsf f := \{e, \mathsf D, \mathsf g, \mathsf h \},\\
\mathsf f \in \mathcal F
}}
\left|
f(X) - \mathsf f (X)
\right|
\\
& \leq
\max_{\substack{
D, g, h, \\
f := \{e, D, g, h\},\\
f \in \mathcal F,\\
X \in \mathcal X
}}
\min_{\substack{
\mathsf D \\
\mathsf f := \{e, \mathsf D, g, h \},\\
\mathsf f \in \mathcal F
}}
\left|
h(g^{\otimes D(X)}(e_1, \ldots, e_L)
- h(g^{\otimes D(X)}(\Tilde{e}_1, \ldots, \Tilde{e}_L)
\right|
\\
& \leq
\gamma \sum_{i \in \iset{ L }} \delta_i \| e_i - \Tilde{e}_i \|,
\end{align}
where the last inequality is an application of Lemma~\ref{alem:fcomp-sensitivity}.

With $C = \max\limits_{\substack{x, x' \in \mathcal I\\ j \in \iset{ L }}} \| e(x, j) - e(x', j) \|$, we get an upper bound of $\gamma C  \sum_{i \in \iset{ L }} \delta_i  = \gamma C \bdelta / \bbeta$, giving us the upper bound in the statement of the theorem with $C_u = \gamma C$.

Focusing on the lower bound, we have the following:

\begin{align}
&
\max_{\substack{
D, g, h, \\
f := \{e, D, g, h\},\\
f \in \mathcal F,\\
X \in \mathcal X
}}
\min_{\substack{
    \mathsf D, \mathsf g, \mathsf h, \\
\mathsf f := \{e, \mathsf D, \mathsf g, \mathsf h \},\\
\mathsf f \in \mathcal F
}}
\left|
f(X) - \mathsf f (X)
\right|
\\
&
=
\max_{\substack{
D, g, h, \\
f := \{e, D, g, h\},\\
f \in \mathcal F,\\
X \in \mathcal X
}}
\min_{\substack{
\mathsf D, \mathsf g, \mathsf h, \\
\mathsf f := \{e, \mathsf D, \mathsf g, \mathsf h \},\\
\mathsf f \in \mathcal F
}}
\left|
h(g^{\otimes D(X)}(e(x_1, 1), \ldots, e(x_L, L)))
- \mathsf h(\mathsf g^{\otimes \mathsf D}(e(x_1, 1), \ldots, e(x_L, L)))
\right|
\end{align}

Now, with assumption (A2), we have for any $\mathsf D, \mathsf g, \mathsf h$

\begin{align}
&
\left|
h(g^{\otimes D(X)}(e(x_1, 1), \ldots, e(x_L, L)))
- \mathsf h(\mathsf g^{\otimes \mathsf D}(e(x_1, 1), \ldots, e(x_L, L)))
\right|
\\
& =
\left|
h(g^{\otimes D(X)}(e(x_1, 1), \ldots, e(x_L, L)))
- h(g^{\otimes \mathsf D}(e(x_1, 1), \ldots, e(x_L, L)))
\right|
\\
& =
\left|
h(g^{\otimes D(X)}(e_1, \ldots, e_L))
- h(g^{\otimes \mathsf D}(e_1, \ldots, e_L))
\right|,
\label{aeq:lb-1}
\end{align}
where $e_i = e(x_i, i)$. Using (A1), and the technique used for the upperbound, we can create a reordering $\Tilde{E} = [\Tilde{e}_1, \ldots, \Tilde{e}_L]$ of the sequence $E = [e_1, \ldots, e_L]$ such that
\begin{equation}
h(g^{\otimes D(X)}(\Tilde{e}_1, \ldots, \Tilde{e}_L)
= h(g^{\otimes \mathsf D}(e_1, \ldots, e_L)).
\end{equation}

Substituting above in \eqref{aeq:lb-1}, we have
\begin{align}
&
\left|
h(g^{\otimes D(X)}(e(x_1, 1), \ldots, e(x_L, L)))
- \mathsf h(\mathsf g^{\otimes \mathsf D}(e(x_1, 1), \ldots, e(x_L, L)))
\right|
\\
& =
\left|
h(g^{\otimes D(X)}(e_1, \ldots, e_L))
- h(g^{\otimes D(X)}(\Tilde{e}_1, \ldots, \Tilde{e}_L))
\right|,
\label{aeq:lb-2}
\\
& \geq
\left|
h(g^{\otimes D(X)}(e_1, \ldots, e_{i^\star - 1}, e_{i^\star}, e_{i^\star + 1}, \ldots, e_L))
- h(g^{\otimes D(X)}(e_1, \ldots, e_{i^\star -1}, \Tilde{e}_{i^\star}, e_{i^\star+1} \ldots, e_L))
\right|,
\label{aeq:lb-3}
\end{align}
where $i^\star$ is the index with the highest LoI for $D(X)$, and the last inequality stems from the fact that we are only perturbing a single index in \eqref{aeq:lb-3} instead of all indices as in \eqref{aeq:lb-2}.

Given that the two terms in the right-hand-side of \eqref{aeq:lb-3} can be written as $f(X)$ and $f(\Tilde{X}^{i^\star})$, we can apply assumption (A3) to have

\begin{equation}
\max_{\substack{
D, g, h, \\
f := \{e, D, g, h\},\\
f \in \mathcal F,\\
X \in \mathcal X
}}
\min_{\substack{
    \mathsf D, \mathsf g, \mathsf h, \\
\mathsf f := \{e, \mathsf D, \mathsf g, \mathsf h \},\\
\mathsf f \in \mathcal F
}}
\left|
f(X) - \mathsf f (X)
\right|
\geq  \eqref{aeq:lb-3} = \max_{X \in \mathcal X} \kappa \delta_{i^\star} \| e_{i^\star} - \Tilde{e}_{i^\star} \|.
\end{equation}

With the maximization over $X$, we can select $X$ such that the $ \| e_{i^\star} - \Tilde{e}_{i^\star} \|$ is greater than some universal constant $C_1$. Then, setting $C_l$ in the statement of the theorem to $C_l = C_1 \kappa$ gives us the result since $\delta_{i^\star} = \bdelta$.
\end{proof}

\clearpage
\section{Technical Details for Theorem~\ref{thm:model:sysgen}}
\label{asec:sysgen-theory}

Systematic generalization is often described as being able to ``handle unknown combinations of known parts''. Here, we will specify a data generating function and data distribution that enables us to study systematic generalization. Usually, systematic generalization is often empirically evaluated in a setup where there are cleanly separable parts of the input, and the model is exposed to the individual parts but not all combinations. For an example, see Hupkes et al.~\citeyear[\S 3.1.1]{hupkes2020compositionality}. This corresponds to the situation where the maximum outgoing degree of the \cdag in the ground-truth data-generating function $f: \mathcal X \to \mathcal Y$ is one (that is, $q=1$).

For the sake of simplicity of exposition, we will also consider the case where the maximum incoming degree in the \cdag of $f$ is two (that is, $k=2$), and there is exactly one sink node (that is, $m=1$). The results we present here can be generalized to any reasonable $k>2$ and $m>1$ with additional technical steps. For generalization guarantees, we have to consider a specific output space $\mathcal Y$; here we will study a bounded regression problem and assume that $\mathcal Y = [0,1] \subset \mathbb R$ without loss of generality.

Now, we need to precisely define the notion of cleanly-separable parts in any $X \in \mathcal X$ to study systematic generalization.

\begin{definition}[Parts of the input]\label{adef:sysgen:parts}
For any input $X = [x_1, \ldots, x_L] \in \mathcal X$, we will define two parts $\mathbf a (X) = \{x_{i_1}, \ldots, x_{i_{n_a}}\}$ and $\mathbf b (X) = \{x_{j_1}, \ldots, x_{j_{n_b}} \}$, where $i_n, j_{n'} \in \iset{ L } \forall n \in \iset{ n_a }, \forall n' \in \iset{ n_b }$, and thus each of the $x_{i_n} \in \mathbf a (X)$ and $x_{j_{n}} \in \mathbf b (X)$ are tokens in $X$. Let $\mathbf c (X) = \{ x_i, i \in \iset{ L }: x_i \notin \mathbf a (X) \cup \mathbf b (X) \}$ be the set of remaining tokens in any input sequence $X$ beyond the parts $\mathbf a (X)$ and $\mathbf b (X)$ that are necessary to make a valid input sequence $X \in \mathcal X$.
\end{definition}

\begin{definition}[Processing of cleanly-separable parts]\label{adef:sysgen:process}
For any input $X \in \mathcal X$ with parts $(\mathbf a(X), \mathbf b(X), \mathbf c(X))$, there exists nodes $\mathbf N_a, \mathbf N_b$ and $\mathbf N$ in the \cdag $D(X)$ of $X$ such that, the sub-{\cdag} rooted at $\mathbf N_a$ and $\mathbf N_b$ only contains source nodes corresponding to input tokens in $\mathbf a(X)$ and input tokens in $\mathbf b(X)$ respectively, and $\{\mathbf N_a, \mathbf N_b\}$ forms the parent set of $\mathbf N$.
\end{definition}

\cref{afig:sysgen:parts} visualizes some example {\cdag}s, highlighting the different parts and corresponding nodes detailed in \cref{adef:sysgen:parts,adef:sysgen:process}. 
Given that we are considering the case where the outgoing degree $q=1$ (that is, the maximum outgoing degree is 1), the above setup implies that there is no overlap between the sub-parts $\mathbf a(X)$ and $\mathbf b(X)$ for any input $X$ (that is $\mathbf a(X) \cap \mathbf b(X) = \emptyset$).

\begin{figure}[t]
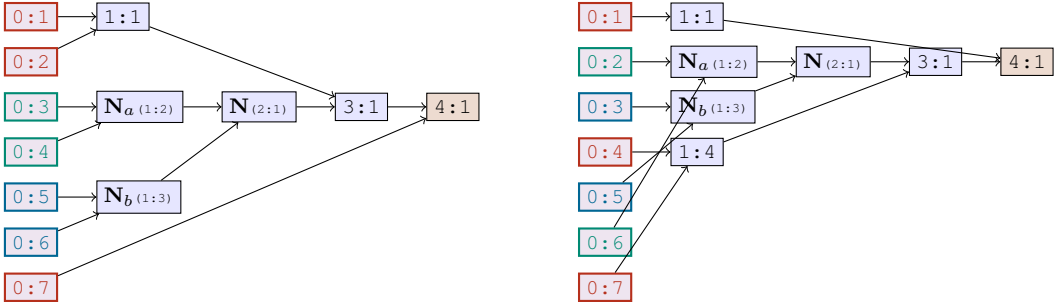

\centering
{\footnotesize {\tt \tikz
  \graph[
    nodes={draw, fill=blue!10},
    grow right sep=0.2in,
    group shift=(-90:0.6),
  ]
  {
    a [as=0:1, BrickRed, thick, fill=Fuchsia!10]
    -> b [as=1:1]
    ;
    c [as=0:2, BrickRed, thick, fill=Fuchsia!10]
    -> b
    ;
    d [as=0:3, PineGreen, thick, fill=Fuchsia!10]
    -> e [as={$\mathbf N_a${\tiny (1:2)}}]
    -> f [as={$\mathbf N${\tiny (2:1)}}]
    -> g [as=3:1]
    -> h [as=4:1, fill=Sepia!10];
    j [as=0:4, PineGreen, thick, fill=Fuchsia!10]
    -> e;
    m [as=0:5, MidnightBlue, thick, fill=Fuchsia!10]
    -> n [as={$\mathbf N_b${\tiny (1:3)}}] -> f;
    o [as=0:6, MidnightBlue, thick, fill=Fuchsia!10]
    -> n;
    p [as=0:7, BrickRed, thick, fill=Fuchsia!10]
    -> h;
    b -> g;
  };
}}
\hskip 30pt
{\footnotesize {\tt \tikz
  \graph[
    nodes={draw, fill=blue!10},
    grow right sep=0.2in,
    group shift=(-90:0.6),
  ]
  {
    a [as=0:1, BrickRed, thick, fill=Fuchsia!10]
    -> aa [as=1:1];
    b [as=0:2, PineGreen, thick, fill=Fuchsia!10]
    -> bb[as={$\mathbf N_a${\tiny (1:2)}}]
    -> bbb[as={$\mathbf N${\tiny (2:1)}}]
    -> bbbb[as=3:1]
    -> bbbbb [as=4:1, fill=Sepia!10];
    c [as=0:3, MidnightBlue, thick, fill=Fuchsia!10]
    -> cc[as={$\mathbf N_b${\tiny (1:3)}}]
    -> bbb;
    d [as=0:4, BrickRed, thick, fill=Fuchsia!10]
    -> dd [as=1:4]
    -> bbbb;
    e [as=0:5, MidnightBlue, thick, fill=Fuchsia!10]
    -> cc;
    f [as=0:6, PineGreen, thick, fill=Fuchsia!10]
    -> bb;
    g [as=0:7, BrickRed, thick, fill=Fuchsia!10]
    -> dd;
    aa -> bbbbb;
  };
}}

\caption{Examples of {\cdag}s highlighting the notion of cleanly-separable parts as presented in \cref{adef:sysgen:parts,adef:sysgen:process}. Source nodes corresponding to the \textcolor{PineGreen}{part $\mathbf a(X)$ are bordered in Green}; source nodes corresponding to the \textcolor{MidnightBlue}{part $\mathbf b(X)$ are bordered in Blue}; the remaining source nodes, corresponding to the \textcolor{BrickRed}{remainder $\mathbf c(X)$, are bordered in Red}. While the example on the left shows parts $\mathbf a(X) = \{x_3, x_4\}$ and $\mathbf b(X) = \{ x_5, x_6 \}$ constituting of consecutive input tokens/source nodes (or a contiguous span of tokens), the example on the right shows that our definition also allows for parts constituting non-contiguous input tokens with $\mathbf a(X) = \{ x_2, x_6 \}$ and $\mathbf b(X) = \{ x_3, x_5\}$. In both these cases, $k = 2, q = 1, m = 1$ and the sequence length $L = 7$.}
\label{afig:sysgen:parts}
\end{figure}

\begin{definition}[Goal of systematic generalization]\label{adef:sysgen:goal}
Let $\mathbf a(X) \in A$ and $\mathbf b(X) \in B$ be elements of types $T_A$ and $T_B$, where $A$ is the set of all sub-parts $\mathbf a(X)$ of any $X \in \mathcal X$ of type $T_A$ over all $X \in \mathcal X$, and $B$ is the set of all sub-parts $\mathbf b(X)$ of any $X \in \mathcal X$ for type $T_B$ over all $X \in \mathcal X$. To study systematic generalization, we study the ability of a model trained with inputs containing sub-parts $a \in A$ of type $T_A$ and $b \in B$ of type $T_B$ (but not all combinations of $a$ and $b$) to generalize to all combinations of sub-parts $(a,b) \in A \times B$.
\end{definition}

\begin{definition}[Training set for systematic generalization]\label{adef:sysgen:train}
For the training data, we will generating samples $(X, y), X_i \in \mathcal X, y = f(X) \in [0,1]$, where $f$ is the ground-truth compositional function with components $(e, D, g, h)$, and the input sequence $X$, with cleanly-separable parts $(\mathbf a(X), \mathbf b(X), \mathbf c(X))$ (as in Definitions~\ref{adef:sysgen:parts} \& \ref{adef:sysgen:process}) is sampled from a distribution $\mathcal P$ (that is, $X \sim \mathcal P$). The training set will constitute $N$ samples $S = \{(X_1, y_1), \ldots, (X_N, y_N) \}$.
\end{definition}

\begin{assumption}[Distributional assumption]\label{aasp:sysgen:dist}
We will assume the following structure on the distribution $\mathcal P$: With a large enough training set $S \sim \mathcal P^n$ of size $n$, for any test point $X \sim \mathcal P$ with parts $(\mathbf a(X), \mathbf b(X), \mathbf c(X)), \mathbf a(X) \in A, \mathbf b(X) \in B$  (as in Definition~\ref{adef:sysgen:parts}), there exists samples $(X', y'), (X'', y''), (X''', y''')$ in the training set $S$ such that, (i)~$X'$ has parts $(\mathbf a(X'), \mathbf b(X'), \mathbf c(X'))$ with $\mathbf a(X) = \mathbf a(X')$ and $\mathbf b(X') \in B$, (ii)~$X''$ has parts $(\mathbf a(X''), \mathbf b(X''), \mathbf c(X''))$ with $\mathbf b(X) = \mathbf b(X'')$ and $\mathbf a(X'') \in A$, and (iii)~$X'''$ has parts $(\mathbf a(X'''), \mathbf b(X'''), \mathbf c(X'''))$ with $\mathbf c(X''') = \mathbf c(X)$ and $\mathbf a(X''') \in A, \mathbf b(X''') \in B$, and the {\cdag}s $D(X)$ and $D(X''')$ are the same barring sub-{\cdag}s rooted at $\mathbf N$ (the child of $(\mathbf N_a, \mathbf N_b)$) and at $\mathbf N'''$ (the child of $(\mathbf N_a''', \mathbf N_b''')$) respectively; note that we do not require $\mathbf a(X''') = \mathbf a(X)$ or $\mathbf b(X''') = \mathbf b(X)$.
\end{assumption}

This assumption explicitly states the usual setup for systematic generalization where we are ensuring that we are evaluating generalization over known parts -- that is, $\mathbf a(X) = \mathbf a(X')$ and $\mathbf b(X) = \mathbf b(X'')$ have been seen in the training set -- and their unknown combination in $X$ -- with parts $(\mathbf a(X), \mathbf b(X), \mathbf c(X))$. The requirement in Assumption~\ref{aasp:sysgen:dist} on $\mathbf c(X) = \mathbf c(X''')$ (the remainder) ensures that the combination of $\mathbf c(X''')$ with parts $\mathbf a(X'''), \mathbf b(X''')$ of types $T_A, T_B$ respectively is seen in the training set. This whole setup is visualized with the \cdag of an example test input $X$ in \cref{afig:sysgen:dist}, with the \cdag of $X$ visualized in \cref{afig:sysgen:dist-X} and the {\cdag}s of the related training examples $X', X'', X'''$ visualized in \cref{afig:sysgen:dist-X1,afig:sysgen:dist-X2,afig:sysgen:dist-X3}. This assumption can potentially be satisfied with $n = O(|A| + |B|)$ which is significantly smaller than the total number of unknown combinations, which is $O(|A|\cdot |B|)$.

\begin{figure}[t]
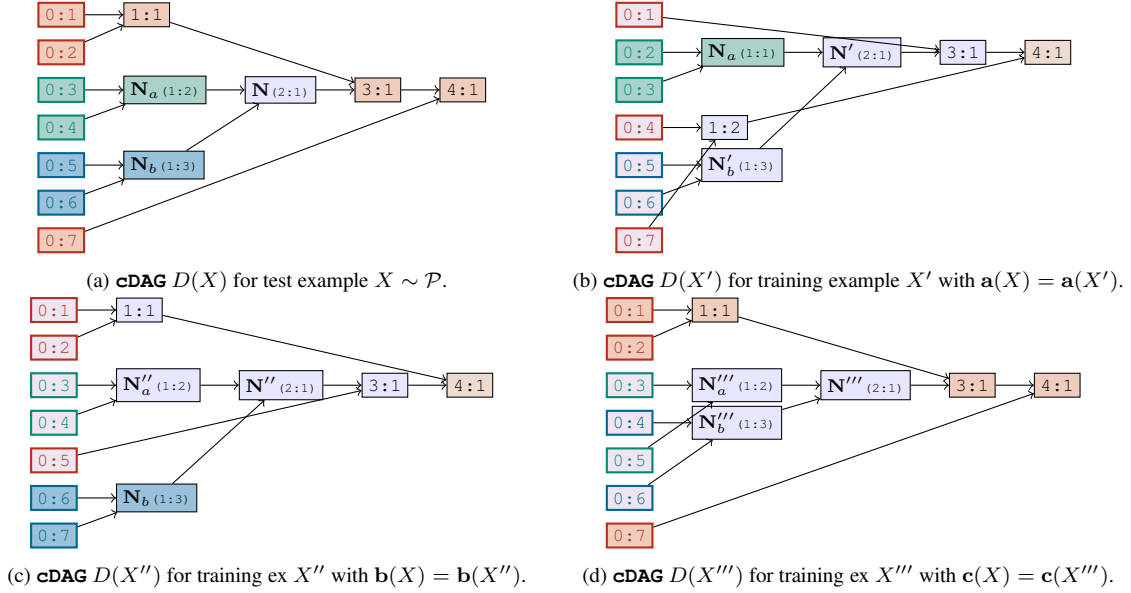

\centering
\begin{subfigure}{0.45\textwidth}
\centering
{\scriptsize {\tt \tikz
  \graph[
    nodes={draw, fill=blue!10},
    grow right sep=0.2in,
    group shift=(-90:0.5),
  ]
  {
    a [as=0:1, BrickRed, thick, fill=BrickRed!20]
    -> b [as=1:1, fill=BrickRed!20]
    ;
    c [as=0:2, BrickRed, thick, fill=BrickRed!20]
    -> b
    ;
    d [as=0:3, PineGreen, thick, fill=PineGreen!30]
    -> e [as={$\mathbf N_a${\tiny (1:2)}}, fill=PineGreen!30]
    -> f [as={$\mathbf N${\tiny (2:1)}}]
    -> g [as=3:1, fill=BrickRed!20]
    -> h [as=4:1, fill=BrickRed!20];
    j [as=0:4, PineGreen, thick, fill=PineGreen!30]
    -> e;
    m [as=0:5, MidnightBlue, thick, fill=MidnightBlue!30]
    -> n [as={$\mathbf N_b${\tiny (1:3)}}, fill=MidnightBlue!30] -> f;
    o [as=0:6, MidnightBlue, thick, fill=MidnightBlue!30]
    -> n;
    p [as=0:7, BrickRed, thick, fill=BrickRed!20]
    -> h;
    b -> g;
  };
}}
\caption{\cdag $D(X)$ for test example $X \sim \mathcal P$.}
\label{afig:sysgen:dist-X}
\end{subfigure}
~
\begin{subfigure}{0.45\textwidth}
\centering
{\scriptsize {\tt \tikz
  \graph[
    nodes={draw, fill=blue!10},
    grow right sep=0.2in,
    group shift=(-90:0.5),
  ]
  {
    a [as=0:1, BrickRed, thick, fill=Fuchsia!10];
    d [as=0:2, PineGreen, thick, fill=PineGreen!30]
    -> e [as={$\mathbf N_a${\tiny (1:1)}}, fill=PineGreen!30]
    -> f [as={$\mathbf N'${\tiny (2:1)}}]
    -> g [as=3:1]
    -> h [as=4:1, fill=Sepia!10];
    j [as=0:3, PineGreen, thick, fill=PineGreen!30]
    -> e;
    b [as=0:4, BrickRed, thick, fill=Fuchsia!10]
    -> c[as=1:2] -> h;
    m [as=0:5, MidnightBlue, thick, fill=Fuchsia!10]
    -> n [as={$\mathbf N_b'${\tiny (1:3)}}] -> f;
    o [as=0:6, MidnightBlue, thick, fill=Fuchsia!10]
    -> n;
    p [as=0:7, BrickRed, thick, fill=Fuchsia!10]
    -> c;
    a -> g;
  };
}}
\caption{\cdag $D(X')$ for training example $X'$ with $\mathbf a(X) = \mathbf a(X')$.}
\label{afig:sysgen:dist-X1}
\end{subfigure}
~
\begin{subfigure}{0.45\textwidth}
\centering
{\scriptsize {\tt \tikz
  \graph[
    nodes={draw, fill=blue!10},
    grow right sep=0.2in,
    group shift=(-90:0.5),
  ]
  {
    a [as=0:1, BrickRed, thick, fill=Fuchsia!10]
    -> b [as=1:1]
    ;
    c [as=0:2, BrickRed, thick, fill=Fuchsia!10]
    -> b;
    d [as=0:3, PineGreen, thick, fill=Fuchsia!10]
    -> e [as={$\mathbf N_a''${\tiny (1:2)}}]
    -> f [as={$\mathbf N''${\tiny (2:1)}}]
    -> g [as=3:1]
    -> h [as=4:1, fill=Sepia!10];
    j [as=0:4, PineGreen, thick, fill=Fuchsia!10]
    -> e;
    p [as=0:5, BrickRed, thick, fill=Fuchsia!10]
    -> g;
    m [as=0:6, MidnightBlue, thick, fill=MidnightBlue!30]
    -> n [as={$\mathbf N_b${\tiny (1:3)}}, fill=MidnightBlue!30] -> f;
    o [as=0:7, MidnightBlue, thick, fill=MidnightBlue!30]
    -> n;
    b -> h;
  };
}}
\caption{\cdag $D(X'')$ for training ex $X''$ with $\mathbf b(X) = \mathbf b(X'')$.}
\label{afig:sysgen:dist-X2}
\end{subfigure}
~
\begin{subfigure}{0.45\textwidth}
\centering
{\scriptsize {\tt \tikz
  \graph[
    nodes={draw, fill=blue!10},
    grow right sep=0.2in,
    group shift=(-90:0.5),
  ]
  {
    a [as=0:1, BrickRed, thick, fill=BrickRed!20]
    -> b [as=1:1, fill=BrickRed!20]
    ;
    c [as=0:2, BrickRed, thick, fill=BrickRed!20]
    -> b
    ;
    d [as=0:3, PineGreen, thick, fill=Fuchsia!10]
    -> e [as={$\mathbf N_a'''${\tiny (1:2)}}]
    -> f [as={$\mathbf N'''${\tiny (2:1)}}]
    -> g [as=3:1, fill=BrickRed!20]
    -> h [as=4:1, fill=BrickRed!20];
    j [as=0:4, MidnightBlue, thick, fill=Fuchsia!10]
    -> n [as={$\mathbf N_b'''${\tiny (1:3)}}] -> f;
    m [as=0:5, PineGreen, thick, fill=Fuchsia!10]
    -> e;
    o [as=0:6, MidnightBlue, thick, fill=Fuchsia!10]
    -> n;
    p [as=0:7, BrickRed, thick, fill=BrickRed!20]
    -> h;
    b -> g;
  };
}}
\caption{\cdag $D(X''')$ for training ex $X'''$ with $\mathbf c(X) = \mathbf c(X''')$.}
\label{afig:sysgen:dist-X3}
\end{subfigure}
\caption{Examples of {\cdag}s as per the distributional Assumption~\ref{aasp:sysgen:dist}. Nodes corresponding to the \textcolor{PineGreen}{part $\mathbf a(X)$ and node $\mathbf N_a$ are in Green}; nodes corresponding to the \textcolor{MidnightBlue}{part $\mathbf b(X)$ and node $\mathbf N_b$ are in Blue}; the remaining nodes, corresponding to the \textcolor{BrickRed}{remainder $\mathbf c(X)$, are in Red}. The input $X$ has parts $\mathbf a(X) = \{x_3, x_4\}$ and $\mathbf b(X) = \{ x_5, x_6 \}$, with the remainder $\mathbf c(X) = \{x_1, x_2, x_7\}$. The test input $X$ corresponds to the ``unknown combination'' of ``known parts'', with the known parts $\mathbf a(X) = \mathbf a(X')$ (\cref{afig:sysgen:dist-X1}), $\mathbf b(X) = \mathbf b(X'')$ (\cref{afig:sysgen:dist-X2}), and $\mathbf c(X) = \mathbf c(X''')$ (\cref{afig:sysgen:dist-X3}) highlighted with matching colors. In all cases, $k = 2, q = 1, m = 1$ and the sequence length $L = 7$.}
\label{afig:sysgen:dist}
\end{figure}

\begin{definition}[Learning setup]\label{adef:sysgen:learning}
We consider a learning setup where we are provided with the ground-truth token encoder $e: \mathcal I \times \mathbb N \to \mathcal H \subset \mathbb R^d$ and the ground-truth \cdag function $D: \mathcal X \to \mathcal D$, and the learning algorithm $\mathcal A$ has to the learn a span encoder $\hat g_S: \mathcal H^2 \to \mathcal H$ and a readout function $\hat h_S: \mathcal H \to [0,1]$ to give us the model $\hat f_S \triangleq (e, D, \hat g_S, \hat h_S)$ using a training set $S$ by minimizing the population risk $R_N(\hat f) = (1/N) \sum_{(X,y) \in S} \ell(y, \hat f(X))$, where $\ell(y, y') = |y - y'|$ is absolute loss. Let $\gamma > 0$ be the Lipschitz constant of the learned readout function $\hat h_S$ for any $S$.
\end{definition}

\begin{assumption}[Low training error]\label{aasp:sysgen:lowtrainerror}
Under the conditions of the above learning setup (as in Definition~\ref{adef:sysgen:learning}), we assume that the learning algorithm $\mathcal A$ is able to learn a model $\hat f_S$ from the training set $S$ such that, $\forall (X, y) \in S, \ell(y, \hat f(X)) \leq \epsilon$, for a small $\epsilon > 0$.
\end{assumption}

As we are making the problem setup more precise, we have made various simplifications (such as $q=1$, $k=2$, and learning only the span encoder and readout functions). While these conditions do not capture all learning scenarios, we would first like to highlight that we have not seen compositional generalization being studied in a theoretically precise manner in existing literature, and the precise definition of the setup is itself one of our contributions. Secondly, even in this simplified setup, we will highlight from our subsequent result on generalization that our proposed notion of compositional complexity plays a critical role in generalization even in this simplified setup. We plan to study more general setups in future work (such as $k>1$, learning the \cdag function $D$, studying non-cleanly-separable sub-problems), and we anticipate that the role of the compositional complexity will become even more acute in those setups.

Now we present our systematic generalization guarantee:
\begin{theorem}\label{athm:sysgen:main}
Consider a $(k=2,q=1,m=1,\bdelta, \bbeta)$-compositional function class $\mathcal F$ (Definition~\ref{def:func-class}), and a learning algorithm $\mathcal A$ that obtains a function $\hat f_S \in \mathcal F$ using a training set $S \sim \mathcal P^N$ of size $N$ sampled from a distribution $\mathcal P$ with a ground-truth data-generating function $f: \mathcal X \to [0,1]$. Under the conditions of above Definitions~\ref{adef:sysgen:parts}, \ref{adef:sysgen:process}, \ref{adef:sysgen:train}, \ref{adef:sysgen:learning}, and Assumptions~\ref{aasp:sysgen:dist}, \ref{aasp:sysgen:lowtrainerror}, with probability at least $1 - (\xi+\xi')$ for some $\xi, \xi' \in (0, 1)$ with $(\xi+\xi') < 1$, we can provide the following systematic generalization guarantee:
\begin{equation} \label{aeq:sysgen:main}
\left | R(\hat f_S) - R_N(\hat f_S) \right|
\leq C_N \gamma \bdelta + (2 N C_N\gamma \bdelta + 1) \sqrt{\frac{2 \log (2/\xi)}{N}},
\end{equation}
where $R(\hat f) = \mathbb E_{X \sim \mathcal P} \ell(f(X), \hat f(X))$ is the true risk of a learning model $\hat f$, and $C_N$ is a positive quantity that depends on $N$ and $\xi'$.
\end{theorem}

To prove this result, we will use the following definition of uniform stability~\cite[Definition 1]{bousquet2000algorithmic} and the related generalization guarantee~\cite[Theorem 2]{bousquet2000algorithmic}, restated using our notation:

\begin{definition}[Uniform stability \cite{bousquet2000algorithmic}] \label{adef:sysgen:unistab}
Let $S = \{(X_1, y_1), \ldots, (X_N, y_N)\} \sim \mathcal P^n \subset (\mathcal X \times [0,1])^N$ be the training set, and $S^i = (S \setminus \{(X_i, y_i)\}) \cup \{(X_i', y_i')\}$ be the training set where the $i$-th training example in $S$ is replaced with a different training example $(X_i', y_i') \sim \mathcal P \in \mathcal X \times [0,1]$ from the same distribution. Let $\hat f_S$ be the model obtained by a symmetric learning algorithm $\mathcal A$ using training set $S$, with $f_{S^i}$ being the model obtained by algorithm $\mathcal A$ on $S^i$. We say that the algorithm $\mathcal A$ is $\sigma$-stable if the following holds:
\begin{equation}\label{aeq:sysgen:unistab}
\forall S \sim \mathcal P^N \forall (X_i', y_i'), (X, y) \sim \mathcal P, \forall i \in \iset{ N }
\left| \ell(y, f_S(X)) - \ell(y, f_{S^i}(X)) \right| \leq \sigma.
\end{equation}
\end{definition}

\begin{theorem}[\cite{bousquet2000algorithmic}] \label{athm:sysgen:unistab-gen}
Let $\mathcal A$ be a $\sigma$-stable algorithm, such that $0 \leq \ell(y, f_S(X)) \leq \mathcal L$ for all $(X, y) \sim \mathcal X \times [0,1]$ and all training sets $S$. For all $\varepsilon > 0$, for any $N > 1$, we have:
\begin{equation}\label{aeq:sysgen:unistab-gen}
\mathbb P_{S \sim \mathcal P^N} \left[
\left|
R_N(f_S) - R(f_S)
\right| > \varepsilon + \sigma
\right] \leq
2 \exp \left( - \frac{ N \varepsilon^2 }{ 2 ( N \sigma + \mathcal L )^2 } \right).
\end{equation}
\end{theorem}

\begin{proof}[Proof of Theorem~\ref{athm:sysgen:main}]
We will essentially leverage Theorem~\ref{athm:sysgen:unistab-gen} by establishing a uniform stability bound for the learning algorithm $\mathcal A$.

For any $(X, y)$, we first have the following inequality

\begin{align}
  \ell(y, \hat f_S(X)) & = |y - \hat f_S(X)| = |y - \hat f_{S^i}(X) + \hat f_{S^i}(X) - \hat f_S(X)| \\
  & \leq |y - \hat f_{S^i}(X) | + |  \hat f_{S^i}(X) - \hat f_S(X) | \\
  & = \ell(y, \hat f_{S^i}(X))  + |  \hat f_{S^i}(X) - \hat f_S(X) | \\
  \Rightarrow
  \ell(y, \hat f_S(X)) - \ell(y, \hat f_{S^i}(X)) & \leq  |  \hat f_{S^i}(X) - \hat f_S(X) |.
\end{align}

Similarly, we can show that

\begin{align}
  \ell(y, \hat f_{S^i}(X)) & \leq \ell(y, \hat f_S(X)) + | \hat f_S(X) - \hat f_{S^i}(X) |,\\
  \Rightarrow
  \ell(y, \hat f_{S^i}(X)) - \ell(y, \hat f_S(X)) & \leq  |  \hat f_{S^i}(X) - \hat f_S(X) |.
\end{align}

Thus, considering the left-hand-side of \eqref{aeq:sysgen:unistab}, we have

\begin{align}
  \left| \ell(y, \hat f_S(X)) - \ell(y, \hat f_{S^i}(X)) \right| 
  & = \max\left(
  \left( \ell(y, \hat f_S(X)) - \ell(y, \hat f_{S^i}(X)) \right),
  \left( \ell(y, \hat f_{S^i}(X)) - \ell(y, \hat f_S(X)) \right)
  \right) \\
  & \leq | \hat f_{S^i}(X) - \hat f_S(X) |.
\end{align}

We know that $\forall y, y' \in [0,1], \ell(y, y') = | y - y' | \leq 1$, thus $\mathcal L = 1$. Now if we can establish universally that $| \hat f_{S^i}(X) - \hat f_S(X) | \leq \sigma$ for some $\sigma$, then we can use Theorem~\ref{athm:sysgen:unistab-gen} with $\mathcal L = 1$ to show that $|R_N(\hat f_S) - R(\hat f_S)| \leq \sigma + \varepsilon$ with probability at least $1 - 2 \exp(-\nicefrac{ N \varepsilon^2 }{ 2 (N \sigma + 1)^2 })$. Setting the failure probability to $\xi \in (0,1)$, we can see that
\begin{equation}
\varepsilon \leq (2N \sigma + 1) \sqrt{\frac{2 \log \nicefrac{ 2 }{ \xi }}{N}},
\end{equation}
and thus, with probability at least $1-\xi$,
\begin{equation} \label{aeq:sysgen:genbound}
  \left | R_N(\hat f_S) - R(\hat f_S) \right| \leq \sigma + (2 N \sigma + 1)  \sqrt{\frac{2 \log \nicefrac{ 2 }{ \xi }}{ N }}
\end{equation}
What remains is to provide a bound for $\sigma$. 

Now, for any $X$ with parts $(\mathbf a(X), \mathbf b(X), \mathbf c(X)), \mathbf a(X) \in A, \mathbf b(X) \in B$ (as in Definition~\ref{adef:sysgen:parts}), and $S, S^i$ satisfying the distributional Assumption~\ref{aasp:sysgen:dist}, we have

\begin{align}
|\hat f_S(X) - \hat f_{S^i}(X) |
& = |\hat f_S(X) - \hat f_S(X''') + \hat f_S(X''') - \hat f_{S^i}(X''') + \hat f_{S^i}(X''') - \hat f_{S^i}(X) | \\
& \leq
|\hat f_S(X) - \hat f_S(X''')| + |\hat f_S(X''') - \hat f_{S^i}(X''')| + |\hat f_{S^i}(X''') - \hat f_{S^i}(X) |, \end{align}
where $(X''', y''')$ is the training sample present both in $S$ and $S^i$ (thus, $X''' \not= X_i$) and $X'''$ has parts $(\mathbf a(X'''), \mathbf b(X'''), \mathbf c(X'''))$, where the remainder $\mathbf c(X''') = \mathbf c(X)$, that is, the remainder part is common between $X$ and $X'''$.

For the second term on the right-hand side $|\hat f_S(X''') - \hat f_{S^i}(X''')|$, we have

\begin{align}
  |\hat f_S(X''') - \hat f_{S^i}(X''')| & = |\hat f_S(X''') - y''' + y''' - \hat f_{S^i}(X''')| \\
  & \leq |y''' - \hat f_S(X''') | + | y''' - \hat f_{S^i}(X''') |
  \leq 2 \epsilon,
\end{align}
where the last inequality is obtained from Assumption~\ref{aasp:sysgen:lowtrainerror} regarding low training error.

Considering the term $| \hat f_S(X) - \hat f_S(X''') |$, we know that, as per Definition~\ref{adef:sysgen:process}, the \cdag $D(X)$ of $X$ contains nodes $\mathbf N_a, \mathbf N_b, \mathbf N$ where node $\mathbf N_a$ (and node $\mathbf N_b$) is the (respective) sink of the sub-{\cdag} with source nodes in $\mathbf a(X)$ (and source nodes in $\mathbf b(X)$), and ${\mathbf N_a, \mathbf N_b}$ are the sole parents of $\mathbf N$. Such similar corresponding nodes $\mathbf N_a''', \mathbf N_b''', \mathbf N'''$ will exist in the \cdag $D(X''')$ of $X'''$.

Barring the sub-{\cdag} $\bar D_{\mathbf N}$ with sink $\mathbf N$ in $D(X)$ and the sub-{\cdag} $\bar D_{\mathbf N'''}$ with sink $\mathbf N'''$ in $D(X''')$, the two {\cdag}s processing $X$ and $X'''$ would be the same as per Assumption~\ref{aasp:sysgen:dist}.

Since the out-going degree of any node in $D(X)$ and $D(X''')$ is one, any difference between $f_S(X)$ and $f_S(X''')$ will start at the start at the corresponding nodes $\mathbf N$ (in $D(X)$) and $\mathbf N'''$ (in $D(X''')$) respectively, and propagate along the path in the respective {\cdag}s to their respective sink nodes $S$ (in $D(X)$) and $S'''$ (in $D(X''')$).

Let $v_\Pi \in \mathcal H$ denote the value of any node $\Pi$ in a \cdag. First, let $P_1, P_2$ be the parents of the sink node $S$ in $D(X)$ and $P_1''', P_2'''$ be the parents of the sink node $S'''$ in $D(X''')$. Since the outgoing degree is one, and $D(X)$ and $D(X''')$ only differ on the parts of the {\cdag}s involving the sub-parts $(\mathbf a(X), \mathbf b(X))$ and $(\mathbf a(X'''), \mathbf b(X'''))$ respectively, one of the two parent nodes of the sink node will have the same value in both the {\cdag}s. We can assume that $v_{P_1} = v_{P_1'''}$ and $v_{P_2} \not= v_{P_2'''}$. Then we have

\begin{align}
| f_S(X) - f_S(X''') |
& = |\hat h_S ( \hat g_S(v_{P_1}, v_{P_2}) ) - \hat h_S( \hat g_S(v_{P_1'''}, v_{P_2'''}) ) | \\
& \leq \gamma c \| v_{P_2} - v_{P_2'''} \|,
\end{align}
where we utilize the Lipschitz constant of $\hat h_S$ and of $g_S$ as presented in Definition~\ref{def:loi}.

Let $\{ P_{21}, P_{22} \}$ be the parent nodes of $P_2$ in \cdag $D(X)$, and $\{P_{21}''', P_{22}'''\}$ be the corresponding parents of $P_2'''$ in \cdag $D(X''')$. As above, given the clean-separability and outgoing degree of one, one of the parents from each of the {\cdag}s will have the same value, and we can assume $v_{P_{21}} = v_{P_{21}'''}$ while $v_{P_{22}} \not= v_{P_{22}'''}$.

\begin{align}
| f_S(X) - f_S(X''') |
& \leq \gamma c \| v_{P_2} - v_{P_2'''} \|
= \gamma c \| \hat g_S (v_{P_{21}}, v_{P_{22}}) - \hat g_S(v_{P_{21}'''}, v_{P_{22}'''}) \| \\
& \leq \gamma c^2 \| v_{P_{22}} - v_{P_{22}'''} \|.
\end{align}

We can continue this recursively till we get to the nodes $\mathbf N$ and $\mathbf N'''$ in the {\cdag}s $D(X)$ and $D(X''')$ respectively after $l$ steps.

Then we have

\begin{align}
| f_S(X) - f_S(X''') |
& \leq \gamma c^l \| v_{\mathbf N} - v_{\mathbf N'''} \|
= \gamma c^l \| \hat g_S (v_{\mathbf N_a}, v_{\mathbf N_b}) - \hat g_S (v_{\mathbf N_a'''}, v_{\mathbf N_b'''}) \|
\end{align}

Consider the pairs $U^\star \triangleq (v_{\mathbf N_a}, v_{\mathbf N_b}) \in \mathcal H^2$ and $U \triangleq (v_{\mathbf N_a'''}, v_{\mathbf N_b'''}) \in \mathcal H^2$, where the latter pair $U$ is seen in the training set with the remainder $\mathbf c(X) = \mathbf c(X''')$. Let $\psi_U$ be the probability measure induced from the training distribution $\mathcal P$ over the pairs of the form $U$, where $\psi_U(B) \in [0,1]$ is the probability measure in the set $B \subset \mathcal H^2$.

Let $\mathcal B(U^\star, r) \subset \mathcal H^2$ denote a ball in $\mathcal H^2$ centered around $U^\star$ with radius $r> 0$. Since we want to find a tight bound on $\| U^\star - U \|$, we essentially want to find the smallest $r$ such that $\mathcal B(U^\star, r)$ will contain at least one $U$ from the training set of the aforementioned form.

Let $\hat \psi_U(B) = \nicefrac{1}{N} \sum_{(X,y) \in S} \mathbb I(U \in B)$ be the empirical estimate of $\psi_U(B)$ ($U$ is derived from $X$). Let us consider a radius $r^\star$ such that the empirical estimate $\hat \psi_U(\mathcal B(U^\star, r^\star)) = \nicefrac{1}{N}$. With Hoeffding's inequality, we can show that, with probability at least $1-\xi', \xi' \in (0, 1)$

\begin{equation}
\left |
\hat \psi_U(\mathcal B(U^\star, r^\star)
- \psi_U(\mathcal B(U^\star, r^\star) 
\right| 
\leq \sqrt{\frac{\log \nicefrac{1}{\xi'}}{N}}.
\end{equation}

Thus, a true measure $\psi_U(\mathcal B(U^\star, r^\star)) = \frac{1}{N} + \sqrt{\frac{\log \nicefrac{1}{\xi'}}{N}}$ will ensure a $U$ from the training set within a radius $r^\star$ of $U^\star$ with high probability. Then 
\begin{equation}
r^\star = \psi_U^{-1} \left( \frac{1}{N} + \sqrt{\frac{\log \nicefrac{1}{\xi'}}{N}} \right) \triangleq r(N, \xi'),
\end{equation}
where $\psi_U^{-1}$ is the inverse measure around $U^\star$. We define the right-hand-side in the above equation as positive function $r(N, \xi')$ that depends on $N$ and $\xi'$. This gives us

\begin{align}
| f_S(X) - f_S(X''') |
& \leq \gamma c^l \| v_{\mathbf N} - v_{\mathbf N'''} \|
= \gamma c^l \| \hat g_S (v_{\mathbf N_a}, v_{\mathbf N_b}) - \hat g_S (v_{\mathbf N_a'''}, v_{\mathbf N_b'''}) \| \\
& \leq \gamma c^{l+1} \left[ \|(v_{\mathbf N_a},  v_{\mathbf N_b}) - (v_{\mathbf N_a'''}, v_{\mathbf N_b'''}) \|\right]
= \gamma c^{l+1} \| U^\star - U \|
\\
& \leq \gamma \, c^{l+1} \, r^\star = \gamma \, c^{l+1} \, r(N, \xi').
\end{align}

Since we want a uniform bound on $|f_S(X) - f_S(X''')|$ for all $X$, $\mathbf N_a, \mathbf N_b$ can be source nodes in the \cdag $D(X)$, with $\mathbf N_a$ and $\mathbf N_b$ having the longest path to the sink node in $D(X)$. In that case,

\begin{align}
| f_S(X) - f_S(X''') |
& \leq \gamma \, \max_l c^{l+1} \, r(N, \xi')
\leq \gamma \, \bdelta \, r(N, \xi'),
\end{align}
where the last inequality is based on the Definition~\ref{def:loi} of the locus of influence in a \cdag with $q=1$, where the $\bdelta = \max_{i \in \iset{ L }} \delta_i$ is for the source node with the longest path to the sink node since there is only a single path from a source node to the sink node.

We can similarly bound $|f_{S^i}(X) - f_{S^i}(X''')|$ and get the following value for the uniform stability bound $\sigma$:

\begin{equation}
\sigma \triangleq 2(\epsilon + \gamma \, \bdelta \, r(N, \xi')).
\end{equation}
For sufficiently small $\epsilon$, we can select a positive quantity $C_N > 0$ which depends on $N$ and $\xi'$, such that $\sigma \leq C_N \, \gamma \, \delta$. Plugging this upper bound of $\sigma$ in \eqref{aeq:sysgen:genbound} gives us \eqref{aeq:sysgen:main} in the statement of the theorem with an overall failure probability of $(\xi+\xi')$.
\end{proof}

\end{document}